\documentclass[11pt]{article}

\usepackage{mathrsfs,amsmath,amsfonts,amssymb,amstext,bm,bbm,dsfont,pifont,amscd,
            amsthm,stmaryrd,euscript,color,xcolor}

\usepackage{algorithmic}
\usepackage{algorithm}
\usepackage{url}
\usepackage{mathtools}
\usepackage{epsfig}
\usepackage{float}
\usepackage{appendix}
\usepackage{amstext}
\usepackage{floatflt}
\usepackage{nicefrac}
\usepackage{amsmath}
\usepackage{anysize,hyperref}
\usepackage{enumerate}
\usepackage{epsfig}
\usepackage{hyperref}
\usepackage{bm}
\usepackage{bbm}
\usepackage{graphicx}
\numberwithin{equation}{section}

\usepackage[a4paper,margin=1in,footskip=0.25in]{geometry}

\input{macros_nicole}

\usepackage{mathtools}

\newtheorem*{theorem*}{Theorem}

\title{Stochastic Gradient Descent Meets Distribution Regression}

\author{Nicole M\"ucke\thanks{MATH+ Junior Research Group 
\emph{Mathematical Foundations of Data Science}}  \\
Technical University  Berlin \\
\texttt{muecke@math.tu-berlin.de} 
}

\date{\today}

\begin{document}

\maketitle

\begin{abstract}
Stochastic gradient descent (SGD) provides a simple and  efficient way to solve
a broad range of machine learning problems. Here, we focus on distribution regression (DR), involving two stages of sampling: 
Firstly, we regress from probability measures to real-valued responses. Secondly, we sample bags from these distributions for utilizing them to 
solve the overall regression problem. 
Recently, DR has been tackled by applying kernel ridge regression and the learning properties of this approach are well understood. 
However, nothing is known about the learning properties of SGD for two stage sampling problems.   
We fill this gap and provide theoretical guarantees  for the performance of SGD for DR. 
Our bounds are optimal in a mini-max sense under standard assumptions. 
\end{abstract}


\section{Introduction}

In a standard non-parametric least squares regression model, the aim is to predict a
response $Y \in \cY$ from a covariate $X$ on some domain $\tilde \cX$. 
Popular approaches are kernel methods \cite{hofmann2008kernel}, where one defines on $\tilde \cX$ 
a reproducing kernel $K$ associated to a \emph{reproducing kernel Hilbert space} $\cH_K$ (RKHS) \cite{aronszajn1950theory, StCh08}. 
The overall aim is to minimize the least squares error over $\cH_K$ by applying a suitable regularization method, involving the kernel and based on 
an i.i.d. sample, drawn according to some unknown distribution on $\tilde \cX \times \cY$. We will later refer to such data as a "first-stage" sample. 

In this paper, we study \emph{distribution regression} (DR) \cite{poczos2013distribution}, where
the covariate is a probability distribution. Typically, we do not observe
this distribution directly, but rather, we observe a "second-stage sample"
drawn from that, amounting to a regression model with measurement error.

Distribution regression has been analyzed in various settings, e.g. 
multiple instance learning \cite{dooly2002multiple,maron1998framework,dietterich1997solving,chevaleyre2001solving,wagstaff2008multiple},  
in an online setting \cite{zhi2013online}, in semi-supervised-learning \cite{zhou2007relation} or active learning \cite{settles2008multiple}.

A popular approach for regression on the domain of distributions
is to embed the distributions into a Hilbert space. This can be achieved by e.g. \emph{kernel mean embeddings} \cite{smola2007hilbert,Muandet_2017}, utilizing another 
appropriate reproducing kernel mapping these distributions into an RKHS. The idea is then to 
introduce the kernel $K$ as a similarity measure between the embedded distributions and  to 
use a traditional kernel machine to solve the overall learning problem.

The learning properties of kernel regularized least squares algorithms based on mean embeddings and with two stages of sampling are rarely analyzed. The first work 
establishing the learning properties of kernel ridge regression (KRR) is \cite{szabo2016learning}, where optimal bounds are derived under suitable assumptions 
on the learning problem and the second-stage sample size. Recently, \cite{fang2020optimal} also considered KRR, with a slight improvement of results. 
However, to the best of our knowledge, an analysis of other kernel based regularization methods is missing. 

While KRR performs an explicit regularization to avoid overfitting, stochastic gradient descent (SGD) performs an implicit regularization as an iterative algorithm. 
Many variants of SGD are known for one-stage least squares regression, 
ranging from considering one pass over the data  \cite{smale2006online,tarres2014online,ying2008online} 
to multiple passes \cite{bertsekas1997new,rosasco2015learning,hardt2016train,LinCamRos16,pillaud2018statistical}, 
with mini-batching \cite{LinRos17} or (tail-)averaging \cite{DieuBa16,mucke2019beating,mucke2020stochastic}.

While SGD is a workhorse in machine learning,  the learning properties of this algorithm in a two stage-sampling 
setting based on mean embeddings are not yet analyzed. We aim at providing an algorithm with reduced computational complexity for two-stage sampling problems, 
compared to KRR, which is known to scale poorly with large sample sizes.

\paragraph{Contributions.}

We analyze the distribution regression problem in the RKHS framework and extend the previous 
approaches in \cite{szabo2016learning} and \cite{fang2020optimal} from two-stage kernel ridge regression to 
two-stage tail-averaging stochastic gradient descent with mini-batching. Our main result is  a computational-statistical efficiency trade-off analysis, 
resulting in finite sample bounds on the excess risk. In particular, we overcome the saturation effect of KRR.  

We give 
a minimum number of the second-stage sample size which is required to obtain the same best possible learning rates as for the classical one-stage SGD algorithm. 
For well-specified models, i.e. the regression function belongs to the RKHS where SGD is performed, we achieve minimax optimal rates with a single pass over the data.  
These bounds match those for classical kernel regularization methods. 
In the mis-specified case, i.e. the regression function does not belong to the RKHS, our bounds also match those 
for the one-stage sample methods with multiple passes over the data.

Moreover, we investigate the interplay of all parameters determining the SGD algorithm, i.e. mini-batch size, step size and stopping time 
and show that the same error bounds can be achieved under various choices of these parameters.
 
On our way we additionally establish the learning properties of tail-averaging two-stage gradient descent which is necessary for deriving our error 
bounds for SGD. Due to space restrictions, this is fully worked out in the Appendix, Section \ref{app:GD}.    

Our results are the first for distribution
regression using SGD and a two-stage sampling strategy.

\paragraph{Outline.} In Section \ref{sec:setup} we introduce the distribution regression problem in detail. 
We introduce our main tool, kernel mean embeddings, and explain the classical non-parametric regression setting in reproducing kernel Hilbert spaces. 
In addition, we define our second-stage SGD estimator.  Section \ref{sec:results} collects our main results for different settings, 
followed by a detailed discussion in Section \ref{sec:discussion}. All proofs are deferred to the Appendix.


\section{The Distribution Regression Problem}
\label{sec:setup}

In this section we introduce the distribution regression problem  in detail. 
Let us begin with with some notation. 
We let $(\cX, \tau)$ be a compact topological space  and denote by $\cB(\cX)$ the Borel $\sigma$-algebra induced by the topology $\tau$. The set 
$\cM^+(\cX)$ denotes the set of Borel probability measures on the measurable space $(\cX, \cB(\cX))$, endowed with the weak topology. 
We furthermore assume that there exists a constant $M>0$ such that $\cY\subseteq [-M, M]$.

Our approach is based on two stages of sampling: 

\begin{enumerate}
\item We are given data $\{(x_j, y_j)\}_{j=1}^n \subset \cM^+(\cX) \times \cY$, i.e., each input $x_j$ is a probability distribution with corresponding label $y_j$. 
Each pair $(x_j, y_j)$ is i.i.d. sampled from a meta distribution $\cM$ on $\cM^+(\cX) \times \cY$.  However, we do not observe $x_j$ directly. 

\item Instead, for each distribution $x_j$ we obtain samples $\{x_{j,i}\}_{i=1}^N \subset \cX$, drawn i.i.d. according to $x_j$. The observed data are 
$\hat \bz= \{(\{x_{j,i}\}_{i=1}^N, y_j)\}_{j=1}^n$. 
\end{enumerate}

\subsection{Our Tool: Kernel Mean Embeddings}

Following the  previous approaches in \cite{szabo2016learning}, \cite{fang2020optimal} we employ \emph{kernel mean embeddings} to map 
the distributions $\{x_j\}_{j=1}^n$ into a Hilbert space. To be more specific, we let $\cH_G$ be a \emph{reproducing kernel Hilbert space} (RKHS) with a Mercer kernel 
$G: \cX \times \cX \to  \mbr$, i.e., $G$ is symmetric, continuous and positive semidefinite \cite{aronszajn1950theory, StCh08}. Moreover, we make the following

\begin{assumption}[Boundedness I]
\label{ass:bounded2}
The kernel $G$ is bounded, i.e. 
\[   \sup_{s \in \cX}G(s, s) =: \gamma^2 < \infty \;, \quad a.s. \;, \]
w.r.t. any probability measure on $\cX$.
\end{assumption}

The associated mean embedding is a map $\mu: \cM^+(\cX) \to \cH_G$, defined as
\begin{equation}\label{eq:KME} 
 \mu_x := \mu(x) := \int_\cX G(s, \cdot)\; dx(s) \;. 
\end{equation} 

Kernel mean embeddings were introduced in e.g. \cite{smola2007hilbert} as a technique   
for comparing distributions without the need for density estimation as an intermediate step and thus have a broad applicability, see also \cite{Muandet_2017} 
and references therein.

Of particular interest are \emph{characteristic} kernels, i.e. the map $\mu: x \mapsto \mu_x$ is injective \cite{fukumizu2004dimensionality}. 
Those type of kernels are essential since $||\mu_x - \mu_{x'}||_{\cH_G}=0$ is equivalent to $x=x'$ and there is no loss of information when mapping a distribution into 
a characteristic RKHS\footnote{A RKHS is called \emph{characteristic} if it's associated kernel is characteristic}.  

It is well known that \emph{universal} kernels\footnote{A continuous kernel is called \emph{universal}, if it's associated RKHS is
dense in the space of continuous bounded functions on the compact domain $\cX$ \cite{StCh08}.} are characteristic, see e.g. 
Theorem 1 in \cite{smola2007hilbert}. Examples include the \emph{exponential kernel, binomial kernel} or the \emph{Gaussian RBF kernel}.  
Thus, a kernel mean embedding serves as a suitable tool for measuring the similarity  between two distributions.

For controlling the two stage sampling process we shall employ this property and compare each distribution $x_j$ from the first stage 
sample with its empirical distribution $\hat x_j:= \frac{1}{N}\sum_{i=1}^N \delta_{x_{j,i}}$ obtained from the second stage of sampling by mapping them 
into the RKHS $\cH_G$ by means of the kernel mean embedding \eqref{eq:KME}.  

Thus, the first stage data for DR are 
\[ D:=\{(\mu_{x_j}, y_j)\}_{j=1}^n \subset \mu(\cM^+(\cX)) \times \cY \;, \]
while the second stage data are 
\[ \hat D:=\{(\mu_{\hat x_j} , y_j)\}_{j=1}^n \subset \mu(\cM^+(\cX)) \times \cY  \;, \]
with the associated mean embeddings  
\[   \mu_{x_j} = \int_\cX G(s, \cdot)\; dx_j(s) \;, \;\;  \mu_{\hat x_j} = \frac{1}{N} \sum_{i=1}^N G(x_{j,i , \cdot }) \;. \]
Both datasets now belong to the same space, making the two stage sampling problem accessible for further investigations applying classical 
kernel methods, as amplified below.

\subsection{One-Stage Least Squares Regression}

We let $\rho$ be a probability measure on $\cZ:=\mu(\cM^+(\cX))\times \cY$ with marginal distribution $\rho_\mu$ on the image $\mu(\cM^+(\cX)) \subset \cH_G$. 
In least squares regression we aim to minimize the risk with respect to the least squares loss, i.e. 
\begin{equation}
\label{eq:min}
 \min_{\cH}\cE(f)\;, \quad \cE(f):= \int_{\cZ} (f(\mu_x)-y)^2 \; d\rho  
\end{equation} 
over a suitable hypotheses space $\cH$. 
Here, we assume that $\cH=\cH_K$ is a RKHS associated with a kernel $K$ on $\mu(\cM^+(\cX))$, satisfying:

\begin{assumption}[Boundedness II]
\label{ass:bounded}
The kernel $K$ is bounded, i.e. 
\[  \sup_{\tilde \mu \in  \mu(\cM^+(\cX))}K( \tilde \mu , \tilde \mu) =: \kappa^2 < \infty \;, \quad \rho_\mu-a.s. \;. \]
\end{assumption}

Note that under this assumption, the RKHS $\cH_K$ can be continuously embedded into  $L^2(\mu(\cM^+(\cX)), \rho_\mu)$ and we henceforth denote this 
inclusion by $S_K: \cH_K \hookrightarrow L^2(\mu(\cM^+(\cX)), \rho_\mu)$.

The minimizer of \eqref{eq:min} over $L^2(\mu(\cM^+(\cX)), \rho_\mu)$ is known to be the \emph{regression function}
\[ f_\rho(\mu_x) = \int_\cY y \;d\rho(y|\mu_x)\;, \quad \mu_x \in \mu(\cM^+(\cX)) \;, \]
where $\rho(\cdot|\mu_x)$ denotes the conditional distribution on $\cY$ given $\mu_x \in \mu(\cM^+(\cX))$. 
Note that our assumption $\cY \subseteq [-M,M]$ implies that $||f_\rho ||_\infty \leq M$.

Classical kernel based approaches for least-squares regression to (approximately) solve \eqref{eq:min} employ some kind 
of explicit or implicit regularization. Among them, and well understood, are Kernel Ridge Regression \cite{optimalratesRLS, fischer2017sobolev}, 
Kernel PCA, Gradient Descent \cite{blanchard2018optimal,lin2020optimal}
or Stochastic Gradient Descent \cite{DieuBa16,LinCamRos16,  LinRos17,mucke2019beating,mucke2020stochastic}. 

All these methods use the first stage data $D=\{(\mu_{x_j}, y_j)\}_{j=1}^n$  to build an estimator $f_D$ with an appropriate amount of regularization 
and the overall aim is to achieve a small \emph{excess risk}
\[  \cE(f_D) - \inf_{f \in \cH}\cE(f) \;,\]
with high probability with respect to the data $D$.



\subsection{Solving DR With Two-Stage Sampling SGD}

Remember we do not directly have access to the first stage data $D$ but by means of the tool of kernel mean embeddings we are able to 
use the second stage data $\hat D$ for our algorithm. Our aim is to perform a variant of stochastic gradient descent. To this end, let 
$i_\cdot=i(\cdot)$ denote a map defining the strategy with which the data are selected at each iteration $t=0,...,T$. The most common approach, which we follow here, 
is sampling each point uniformly at random with replacement. 
We additionally consider \emph{mini-batching}, where  a batch of size $b \in \{1,...,n\}$ of data points  
at each iteration is selected. Formally, the $j_1, ..., j_{bT}$ 
are iid random variables, distributed according to the uniform distribution on $\{1, ..., n\}$.

Starting with $ \hat h_{0} \in \cH_K$, our SGD recursion is given by 
\[ \hat h_{t+1} = \hat h_t - \eta \frac{1}{b}\sum_{i=b(t-1)+1}^{bt} ( \hat h_t(\mu_{\hat x_{j_i}})- y_{j_i})K_{\mu_{\hat x_{j_i}}} \;,  \]
where we write $K_\mu:=K(\mu, \cdot)$ and where $\eta >0$ is the stepsize. The number of passes after $T$ iterations is $\floor{bT/n}$.

We are particularly interested in tail-averaging the iterates, that is 
\begin{equation}
\label{eq:algo}
\bar{ \hat h }_T  := \frac{2}{T}\sum_{t= \floor{T/2} +1}^T \hat h_t \;.
\end{equation}
The idea of averaging the iterates goes back to \cite{RM51}, \cite{PolJu92}, see also \cite{ShaZha13}. 
More recently, in \cite{DieuBa16} full averaging, i.e. summing up the iterates from 
$t=1$ to $t=T$,  was shown to lead to the possibility of choosing larger/ constant  
stepsizes. However, it is also known to lead to \emph{saturation}, i.e. the rates of convergence do not improve anymore in certain 
well-specified cases and thus leads to suboptimal bounds in the high smoothness regime. 
This has been alleviated in \cite{mucke2019beating}  by 
considering tail-averaging, see also \cite{mucke2020stochastic}.

Note that our SGD algorithm only has access to the observed input samples  $\{x_{j,i}\}_{i=1}^N$, $j=1,...,n$ through their 
mean embeddings $\{\mu_{\hat x_j}\}_{j=1}^n$.

\paragraph{Main goals: }We analyze the excess risk\footnote{$\mbe_{\hat D | D}$ denotes the conditional expectation with respect to the sample $\hat D$ given $D$. }  
\[ \mbe_{\hat D |D}[ \cE(\bar{ \hat h }_T ) - \cE(f_\rho) ] = \mbe_{\hat D |D}[ ||S_K \bar{ \hat h }_T - f_\rho ||^2_{L^2} ] \] 
and study the interplay of all parameter $b, \eta, T$ determining the SGD algorithm. We derive finite-sample high probability bounds, 
presenting computational-statistical efficiency trade-offs in our main Theorem \ref{theo:SGD-well-main-body}. In addition, 
we give fast rates of convergence as the sample sizes $n$ grows large and give an answer to the question 

\begin{center}
\emph{ How many second-stage samples N do we need to obtain best possible learning rates, comparable to one-stage learning ? }
\end{center}

Our bounds depend on the 
difficulty of the problem. More precisely, we shall investigate the learning properties of \eqref{eq:algo} in two different basic settings:

{\bf 1. Well-specified Model: } Here, we assume that the regression function $f_\rho$ belongs to the RKHS $\cH_K$. We analyze this setting in 
Section \ref{sec:well} and give high probability bounds, matching the known optimal bounds for one stage regularization methods and two stage 
kernel ridge regression.

{\bf 2. Mis-specified Model: } In this case the regression function is assumed to not to belong to the RKHS $\cH_K$. 
These bounds are presented in Section \ref{sec:mis} and still match the known optimal bounds in the so called \emph{easy learning} regime, to be refined below. 
For so called \emph{hard learning} problems, our bounds still match the best known ones for classical one-stage kernel methods. 




\section{Main Results}
\label{sec:results}

This section is devoted to presenting our main results. Before we go into more detail, we formulate our assumptions on the learning setting. 
The first one considers the reproducing kernel that we define on the set $\mu(\cM^+(\cX))$.

\begin{assumption}[H\"older Property]
\label{ass:lipschitz}
Let $\alpha \in (0,1]$ and $L>0$. We assume that the mapping $K_{(\cdot)}: \mu(\cM^+(\cX)) \to \cH_K$ defined as 
$\tilde \mu \mapsto K(\tilde \mu , \cdot)$ is $(\alpha , L)$-H\"older continuous, i.e. 
\[  || K_{\mu_1}- K_{\mu_2}||_{\cH_K} \leq L ||\mu_1 - \mu_2||^\alpha_{\cH_G} \;,\]
for all $\mu_1, \mu_2 \in  \mu(\cM^+(\cX))$. 
\end{assumption}

The next assumption refers to the \emph{regularity} of the regression function $f_\rho$. It is a well established fact in learning theory that the 
regularity of $f_\rho$ describes the hardness of the learning problem and has an influence of the rate of convergence of any algorithm. To smoothly 
measure the regularity of $f_\rho$ we introduce the 
\emph{kernel integral operator} $L_K=S_KS_K^*: L^2(\mu(\cM^+(\cX)), \rho_\mu) \to L^2(\mu(\cM^+(\cX)), \rho_\mu)$, defined by 
\[ L_K f( \tilde \mu):= \int_{\mu(\cM^+(\cX))} K(\mu', \tilde \mu)f(\mu') \; d\rho_\mu (\mu')\;.  \]
Note that under Assumption \ref{ass:bounded}, $L_K$ is positive, self-adjoint, trace-class and hence compact, with 
$||L_K|| \leq trace(L_K) \leq \kappa^2$, see e.g. \cite{steinwart2012mercer}.

\begin{assumption}[Regularity Condition]
\label{ass:source}
We assume that for some $r>0$ the regression function $f_\rho$ satisfies
\[  f_\rho = L_K^r h_\rho\;, \quad h_\rho \in L^2(\mu(\cM^+(\cX)), \rho_\mu) \;,\]
with $||h_\rho||_{L^2} \leq R$, for some $R< \infty$.
\end{assumption}

This assumption is also known as a \emph{source condition}. We recall here that powers of $L_K$ are defined by spectral calculus, see for instance \cite{reed2012methods}. 
The larger the parameter $r$, the smoother is $f_\rho$. We have the range inclusions 
$\cR ange(L_K^r)\subseteq \cR ange(L_K^{r'})$ if $r \geq r'$ with $\cR ange(L_K^r)\subseteq \cH_K$ for any $r \geq \frac{1}{2}$. Thus, if 
$r\geq \frac{1}{2}$, then $f_\rho$ belongs to $\cH_K$ under Assumption \ref{ass:source} and we are in the well-specified case. 
For more general smoothness assumptions we refer to \cite{mucke2020stochastic}. 


Our last condition refers to the capacity of the RKHS $\cH_K$. Given $\lambda>0$, we define the \emph{effective dimension}
\[ \cN(\lam):= trace\left( L_K(L_K + \lambda)^{-1}\right) \;.\]
This key quantity can be used to describe the complexity of $\cH_K$. 

\begin{assumption}[Effective Dimension]
\label{ass:eigenvalue}
We assume that for some $\nu \in (0,1]$, $c_\nu<\infty$, the effective dimension obeys
\begin{equation}
\label{eq:effdim-main}
 \cN(\lambda ) \leq c_\nu \lambda^{-\nu}  \;. 
\end{equation} 
\end{assumption}

This assumption is common in the nonparametric regression setting, see e.g. \cite{Zhang03} or \cite{optimalratesRLS, lin2020optimal}. 
Roughly speaking, it quantifies how far $L_K$ is from being finite rank. This assumption is satisfied if the eigenvalues $(\sigma_j)_{j \in \mbn}$ 
of $L_K$ have a polynomial decay $\sigma_j \leq c' j^{-\frac{1}{\nu}}$, $c' \in \mbr_+$. Since $L_K$ is trace class, the above assumption is always 
satisfied with  $\nu =1$ and $c_\nu=\kappa^2$.
Smaller values of $\nu$ lead to faster rates of convergence.

Being now well prepared, we state our main result.

\begin{theorem}
\label{theo:SGD-well-main-body}
Suppose Assumptions \ref{ass:bounded2}, \ref{ass:bounded}, \ref{ass:lipschitz}, \ref{ass:source} and \ref{ass:eigenvalue} are satisfied. 
Let further $\delta \in (0,1]$, $\eta < \frac{1}{4\kappa^2}$, $\lambda = (\eta T)^{-1}$ and assume 
\[ n \geq \frac{32\kappa^2 \log(4/\delta)}{\lambda} \log\left(  e\cN(\lambda )\left(1 + \frac{\lambda }{||L_K||} \right)   \right) \;.  \]
Then with probability not less than $1-\delta$, 
the excess risk for the second stage tail-averaging SGD algorithm  \eqref{eq:algo} satisfies 
\begin{align*}
& \mbe_{\hat D | D}[ || S_K\bar{ \hat h }_T - f_\rho ||_{L^2 }] \leq  
   C \;\log(6/\delta)\left(  \; (\eta T)^{-r}   + \sqrt{\frac{ \cN(\lambda )}{n} }  \; +  \;  \frac{(\eta T)^{\frac{1}{2}-r}}{\sqrt{n} }   \right. \\
& \quad + \left. \log(T) \frac{\eta T}{N^{\frac{\alpha}{2}}}\left( 1+ 1_{(0, \frac{1}{2}]}(r) (\eta T)^{\max\{ \nu, 1-2r\}} \right) +  \right. \\
& \quad +  \left.     \sqrt{ \frac{\eta}{b} (\eta T)^{\nu -1}} \;    \left( \frac{\eta T}{\sqrt n} +  \frac{\sqrt{\eta T}}{ N^{\frac{\alpha}{2}}}  \right)^{1/2}   \; \right)\;, 
\end{align*}
for some constant $C < \infty$, depending on the parameters $\gamma, \kappa, \alpha, L, M$, but not on $n$ or $N$. 
\end{theorem}

Note that for the sake of clarity and due to space restrictions we only report the leading error terms. A full statement of this Theorem including all lower 
order terms with it's proof is given in the Appendix, Section \ref{app:results-SGD-second}.

From Theorem \ref{theo:SGD-well-main-body} we can now draw some conclusions. Below, we will give rates of convergence, depending on different 
a priori assumptions on the hardness of the learning problem.


\subsection{Well-specified Case}
\label{sec:well}

Here, we give rates of convergence for the most easiest learning problem where our model is well-specified and the regression function lies in the same space as 
our second-stage SGD iterates, namely in $\cH_K$.

\begin{corollary}[Learning Rates Well-Specified Model]
\label{cor:SGD-well-main-body}
Suppose all assumptions of Theorem \ref{theo:SGD-well-main-body} are satisfied.  Let $r \geq \frac{1}{2}$, $\eta_0 < \frac{1}{4\kappa^2}$ and choose 
\[ N_n\geq  \log^{2/\alpha}(n)\left(\frac{R^2n}{M^2}\right)^{\frac{2r+1}{\alpha(2r+\nu)}}  \;. \] 
Then, for any $n$ sufficiently large, the excess risk satisfies with probability at least $1-\delta$ w.r.t. the data $D$ 
\[  \mbe_{\hat D | D}[ || S_K\bar{ \hat h }_{T_n} - f_\rho ||_{L^2 }]  \leq C \log(6/\delta) R \left(\frac{M^2}{R^2 n}\right)^{\frac{r}{2r + \nu}} \;, \]
for each of the following choices: 
\begin{enumerate}
\item One-pass SGD: $b_n=1$, $\eta_n= \eta_0 \frac{R^2}{M^2}\left(\frac{M^2}{R^2 n}\right)^{\frac{2r+\nu-1}{2r+\nu}}$ and $T_n = \frac{R^2}{M^2}n$, 
\item Early stopping and one-pass SGD:  $b_n=n^{\frac{2r+\nu-1}{2r+\nu}}$, $\eta_n=\eta_0$ and $T_n=\left(\frac{R^2 n}{M^2}\right)^{\frac{1}{2r+\nu}}$, 
\item Batch-GD: $b_n=n$, $\eta_n=\eta_0$ and $T_n=\left(\frac{R^2 n}{M^2}\right)^{\frac{1}{2r+\nu}}$. 
\end{enumerate}
\end{corollary}

We comment on these results in Section \ref{sec:discussion}.


\subsection{Mis-specified Case}
\label{sec:mis}

In this subsection we investigate the mis-specified case and further distinguish between two cases: 

{\bf 1. $r \leq \frac{1}{2}$ but $2r + \nu >1$: } This setting is sometimes called \emph{easy problems}. 

{\bf 2. $r \leq \frac{1}{2}$ but $2r + \nu \leq 1$:} This setting is dubbed \emph{hard problem}, see \cite{pillaud2018statistical}. 

\vspace{0.3cm}

\begin{corollary}[Learning Rates Mis-Specified Model; $2r + \nu >1$]
\label{cor:SGD-mis-main-body1}
Suppose all assumptions of Theorem \ref{theo:SGD-well-main-body} are satisfied.  
Let $r \leq \frac{1}{2}$, $2r + \nu >1$, $\eta_0 < \frac{1}{4\kappa^2}$ and choose 
\begin{equation}
\label{eq:N}
 N_n \geq \log^{2/\alpha}(n)\left(\frac{R^2}{M^2}n \right)^{\frac{2+\nu}{\alpha(2r + \nu)}} \;. 
\end{equation} 
Then, for any $n$ sufficiently large, the excess risk satisfies with probability at least $1-\delta$ w.r.t. the data $D$ 
\[  \mbe_{\hat D | D}[ || S_K\bar{ \hat h }_{T_n} - f_\rho ||_{L^2 }]  \leq C \log(6/\delta) R \left(\frac{M^2}{R^2 n}\right)^{\frac{r}{2r + \nu}} \;, \]
for each of the following choices: 
\begin{enumerate}
\item  Multi-pass SGD: $b_n=\sqrt{n}$, $\eta_n= \eta_0$ and $T_n=\left(\frac{R^2 n}{M^2}\right)^{\frac{1}{2r+\nu}}$, 
\item Batch GD: $b_n=n$, $\eta_n= \eta_0$ and $T_n=\left(\frac{R^2 n}{M^2}\right)^{\frac{1}{2r+\nu}}$.  
\end{enumerate}
\end{corollary}

\begin{corollary}[Learning Rates Mis-Specified Model; $2r + \nu \leq 1$]
\label{cor:SGD-mis-main-body2}
Suppose all assumptions of Theorem \ref{theo:SGD-well-main-body} are satisfied.  Let $K>1$, $r \leq \frac{1}{2}$,  $2r + \nu \leq 1$, 
$\eta_0 < \frac{1}{4\kappa^2}$ and choose  
\[  N_n\geq  \left(\frac{R^2n}{M^2\log^K(n)}\right)^{\frac{3-2r}{\alpha}} \;. \]
Then, for any $n$ sufficiently large, the excess risk satisfies with probability at least $1-\delta$ w.r.t. the data $D$ 
\[  \mbe_{\hat D | D}[ || S_K\bar{ \hat h }_{T_n} - f_\rho ||_{L^2 }]  \leq C \log(6/\delta) R \left( \frac{M^2 \log^K(n)}{R^2n} \right)^r \;, \]
for each of the following choices: 
\begin{enumerate}
\item $b_n=1$, $\eta_n=\left( \frac{M^2\log^K(n)}{R^2n} \right)^{2r+\nu}$ and $T_n=\left( \frac{R^2n}{M^2\log^K(n)} \right)^{2r+\nu+1}$,  
\item $b_n=\left( \frac{R^2n}{M^2\log^K(n)} \right)^{2r+\nu}$, $\eta_n=\eta_0$ and $T_n=\frac{R^2n}{M^2\log^K(n)} $, 
\item $b_n= \frac{R^2n}{M^2\log^K(n)}$, $\eta_n=\eta_0$ and $T_n=\frac{R^2n}{M^2\log^K(n)} $.
\end{enumerate}
\end{corollary}

Again, we comment on these results in Section \ref{sec:discussion} in detail.


\section{Discussion of Results}
\label{sec:discussion}




We now comment on our results in more detail and also compare, in possible cases, to previous results.

\paragraph{High level comments.} Let us briefly describe the nature of our results. In all our bounds above, we are able to establish optimal/ best known 
rates of convergence if the sample size of the second-stage sample is sufficiently large. In Corollary  \ref{cor:SGD-well-main-body}, we 
need 
\[ N_n \geq \log(n)n^{\frac{2r+1}{\alpha(2r+\nu)}} \;. \] 
While choosing a smaller size comes with computational savings, it would reduce the statistical efficiency. In addition, increasing this number beyond this 
value would not lead to any gain in statistical accuracy, but would worsen computational requirements. The same phenomenon is observed in  
Corollary \ref{cor:SGD-mis-main-body1} and  Corollary  \ref{cor:SGD-mis-main-body2}.

We also observe an influence of the degree of smoothness of the kernel applied. 
Choosing a smoother kernel, i.e. a large H\"older index $\alpha \in (0,1]$ reduces the number of samples required, the lowest is achieved for $\alpha =1$.

Finally, smoother regression functions (corresponding to large $r$) are easier to reconstruct, i.e. $N_n$ gets smaller for increasing $r$.


\paragraph{Comparison to one-stage kernel methods.}

Optimal learning bounds for traditional one-stage regularization (kernel) methods are known under various assumptions. 
For "easy learning" problems, i.e. if $2r+\nu>1$, the optimal 
learning rate is of order $\cO(n^{-\frac{r}{2r+\nu}})$ if the amount of regularization is chosen appropriately, 
see \cite{optimalratesRLS}, \cite{lin2020optimal}, \cite{blanchard2018optimal}. 
Our results in Corollary \ref{cor:SGD-well-main-body} and Corollary \ref{cor:SGD-mis-main-body1} match these optimal bounds, provided the number $N_n$
of second-stage samples is chosen sufficiently large, depending on the number of first-stage samples.

For "hard learning" problems, i.e. if $2r+\nu\leq 1$, the best known learning rates for one-stage regularization methods are of order 
$\cO\left(  \left(\frac{\log^K(n)}{n}\right)^r \right)$, $K>1$, see  \cite{fischer2017sobolev}, \cite{lin2020optimal}, \cite{pillaud2018statistical}. 
Our bounds from Corollary \ref{cor:SGD-mis-main-body2} also match this bound if $N_n$ is sufficiently large.


\paragraph{Comparison to two-stage KRR.}

The first paper establishing learning theory for distribution regression using a two stage sampling strategy is 
\cite{szabo2016learning}. In this paper, the authors consider a two-stage kernel ridge regression estimator (KRR) and derive optimal rates in the well-specified case 
$\frac{1}{2}\leq r \leq 1$ if the number of second-stage samples is sufficiently large. More precisely, if 
\[ N_n \geq \log(n)n^{\frac{2r+1}{\alpha(2r+\nu)}} \;, \]
the rate $\cO(n^{-\frac{r}{2r+\nu}})$ given in that paper matches our optimal rate from Corollary \ref{cor:SGD-well-main-body}, under the same number $N_n$. 
However, for mis-specified models, the results in this paper take not the capacity condition \eqref{eq:effdim-main}  
into account\footnote{This amounts to considering the worst case with $\nu=1$.} 
and differ from our bounds. If $0 < r \leq \frac{1}{2}$, the rate obtained is $\cO(n^{-\frac{r}{r+2}})$ if 
\[  N_n \geq \log(n)n^{\frac{2(r+1)}{\alpha(r+2)}}  \;. \] 
Compared to our result in Corollary \ref{cor:SGD-mis-main-body1} with $\nu=1$, this number is smaller that ours in \eqref{eq:N}, but it only gives suboptimal bounds. 
Our result shows that increasing the number of second-stage samples $N_n$ leads to optimal rates also this setting. We also emphasize that KRR suffers from saturation. 
Using tail-ave SGD instead, we can overcome this issue and establish optimality also for $r\geq 1$. 

We also refer to \cite{fang2020optimal} where for KRR in the well specified case $\frac{1}{2} \leq r \leq 1$, the logarithmic pre-factor for $N_n$ 
could be removed.   

However, for the "hard learning" regime, to the best of our knowledge, no learning rates taking Assumption  \eqref{eq:effdim-main}  into account 
are known for two-stage sampling, except our Corollary  \ref{cor:SGD-mis-main-body2}. Thus, we cannot compare our results in this case.


\paragraph{Some additional remarks specific for SGD.} 
Finally, we give some comments specifically related to the SGD algorithm we are applying and compare our results with those known for SGD 
in the one-stage sampling setting. 
In all our results we precisely describe the interplay of all parameters guiding the algorithm: batch-size $b$, stepsize $\eta >0$ 
and stopping time $T$.

All our results show that different parameter choices allow to achieve the same error bound. As noted above, the bound in Corollary \ref{cor:SGD-well-main-body} 
are mini-max optimal, i.e. there exists a corresponding lower 
bound (provided the eigenvalues of $L_K$ satisfy a polynomial lower bound $\sigma_j \geq c j^{-{1/\nu}}$). In addition, these bounds and 
the parameter choices coincide with those in \cite{mucke2019beating}. In particular we achieve statistically optimal bounds with a single pass over the data 
also in the two-stage sampling setting if $f_\rho \in \cH_K$, see Corollary \ref{cor:SGD-well-main-body}, 1. and 2.\;. We also recover the known bound 
for a stochastic version of gradient descent in  Corollary \ref{cor:SGD-well-main-body}, 3, see  \cite{blanchard2018optimal}, \cite{lin2020optimal}.

Moreover, as pointed out in \cite{mucke2019beating}, combining mini-batching with tail-averaging brings some benefits. 
Indeed, in \cite{LinRos17} it is shown that a large stepsize of order $\log(n)^{-1}$ can be chosen if the mini-batch size is of order 
$b_n=\cO(n^{\frac{2r}{2r+\nu}})$ with a number $\cO(n^{\frac{1}{2r+\nu}})$ of passes. \cite{mucke2019beating} show that with a comparable number of passes 
it is allowed to 
use a larger constant step-size with a much smaller mini-batch size. 
We observe the same phenomenons in the two-stage sampling setting, provided $N_n$ is sufficiently large. 
Finally, Corollary \ref{cor:SGD-well-main-body} also shows that increasing the mini-batch size beyond a critical value does not yield any benefit. 

However, if $f_\rho \not \in \cH_K$ we do not achieve the best known bounds with a single pass and multiple passes are necessary. As in the well-specified 
case, we can achieve these bounds with a large constant stepsize and increasing the mini-batch size beyond a certain value does not yield any benefit. 
Here, we want to stress once more that our results are the first for distribution regression using SGD and a two-stage sampling strategy.



\subsubsection*{Acknowledgements} 
This work is funded by the Deutsche Forschungsgemeinschaft (DFG)
under Excellence Strategy \emph{The Berlin Mathematics Research Center MATH+} 
(EXC-2046/1, project ID:390685689). 

The author is also thankful to three anonymous reviewers who gave useful and kind comments. 



\bibliographystyle{abbrv}
\bibliography{bib_SGD}


\appendix

\section{First Results} 
\label{sec:notation}

We introduce some auxiliary operators, being useful in our proofs. These operators have been introduced in a variety of previous works, 
see e.g. \cite{optimalratesRLS, DieuBa16,  blanchard2018optimal}. Recall that $S_K: \cH_K \hookrightarrow  L^2(\mu(\cM^+(\cX)), \rho_\mu)$ denotes the canonical 
injection map. The adjoint $S_K^*: L^2(\mu(\cM^+(\cX)), \rho_\mu) \to \cH_K$ is given by 
\[  S_K^*g = \int_{\mu(\cM^+(\cX))} g(\tilde \mu)K_{\tilde \mu} \; \rho_\mu(d \tilde \mu)\;,  \]
where we remind at the notation $K_{\tilde \mu} = K(\tilde \mu , \cdot )$.  
The covariance operator is $T_K:=S_K^*S_K:  \cH_K \to \cH_K $, with 
\[  T_K f =    \int_{\mu(\cM^+(\cX))} \inner{ K_{\tilde \mu} , f}_{\cH_K}K_{\tilde \mu} \; \rho_\mu(d \tilde \mu)  \]
and the kernel integral operator $L_K: L^2(\mu(\cM^+(\cX)), \rho_\mu) \to L^2(\mu(\cM^+(\cX)), \rho_\mu) $ is 
\[ L_Kg =    \int_{\mu(\cM^+(\cX))} g (\tilde \mu) K(\tilde \mu , \cdot )\; \rho_\mu(d \tilde \mu)   \;. \]

We further introduce the empirical counterparts: 
\[ T_{\bx} := \frac{1}{n}\sum_{j=1}^n \inner{K_{\mu_{x_j}}, \cdot }_{\cH_K} K_{\mu_{x_j}} \;, \qquad 
   T_{\hat \bx} := \frac{1}{n}\sum_{j=1}^n \inner{K_{\mu_{\hat x_j}}, \cdot }_{\cH_K}  K_{\mu_{\hat x_j}}  \]
\[  g_{\bz }:=   \frac{1}{n}\sum_{j=1}^n y_j  K_{\mu_{x_j} }\;, \qquad   g_{\hat  \bz }:=   \frac{1}{n}\sum_{j=1}^n y_j  K_{\mu_{\hat  x_j}}\;,   \]
where  $T_{\bx}, T_{\hat \bx}: \cH_K \to \cH_K$, $g_{\bz }, g_{\hat \bz } \in \cH_K$.

\[\]

We collect some preliminary results. 

\vspace{0.3cm}

\begin{lemma}
\label{lem:0}
Suppose Assumptions \ref{ass:bounded2}, \ref{ass:bounded} and \ref{ass:lipschitz} are satisfied. Then 
\[  \mbe_{\hat D | D}[ ||T_K( T_{\hat \bx} + \lambda Id )^{-1}|| ] \leq 
\frac{1}{\lambda}||T_K - T_\bx|| +  \frac{c_\alpha \gamma^{\alpha}LM}{\sqrt{\lam} N^{\frac{\alpha}{2}}}  + 1 \;,\]
for some $c_\alpha <\infty$. Moreover, for any $\delta \in (0,1]$, with probability at least $1-\delta$ w.r.t. the data $D$, one has
\[  \mbe_{\hat D | D}[ ||T_K( T_{\hat \bx} + \lambda Id )^{-1}|| ] \leq 
  6\log(2/\delta)\frac{1}{\lambda \sqrt n} +  \frac{c_\alpha \gamma^{\alpha}LM}{\sqrt{\lam} N^{\frac{\alpha}{2}}}  + 1 \;.\]
\end{lemma}

\begin{proof}[Proof of \ref{lem:0}]
Let us bound for any $\lambda >0$
\begin{align*}
 \mbe_{\hat D | D}[ ||T_K( T_{\hat \bx} + \lambda Id )^{-1}|| ] 
&\leq ||T_K(T_K + \lambda Id)^{-1}|| \cdot \mbe_{\hat D | D}[ ||(T_K + \lambda Id)( T_{\hat \bx} + \lambda Id )^{-1}|| ] \\
&\leq \mbe_{\hat D | D}[ ||(T_K + \lambda Id)( T_{\hat \bx} + \lambda Id )^{-1}|| ]\;.
\end{align*}
We proceed by writing 
\begin{align*}
 (T_K + \lambda Id)( T_{\hat \bx} + \lambda Id )^{-1} 
&= (T_K + \lambda Id) (  ( T_{\hat \bx} + \lambda Id )^{-1} - ( T_K + \lambda Id )^{-1}  ) + (T_K + \lambda Id) ( T_K + \lambda Id )^{-1} \\
&=  ( (T_K - T_\bx ) + (T_\bx - T_{\hat \bx})) ( T_{\hat \bx} + \lambda Id )^{-1}  + Id \;.  
\end{align*} 
Since $||( T_{\hat \bx} + \lambda Id )^{-1}|| \leq 1/\lambda$, this leads to 
\begin{align*}
 \mbe_{\hat D | D}[ ||T_K( T_{\hat \bx} + \lambda Id )^{-1}|| ] &\leq \frac{1}{\lambda}||T_K - T_\bx||   + \frac{1}{\lambda} \mbe_{\hat D | D}[ ||T_\bx - T_{\hat \bx}|| ] + 1 \;. 
\end{align*}
The first result follows then from Lemma \ref{lem:2} and Jensen's inequality. 
The second result follows  from the first one by applying Proposition 5.5. in \cite{blanchard2018optimal}. 
\end{proof}

\[\]

\begin{lemma}\cite[Eq. (37)]{fang2020optimal}
\label{lem:1}
Suppose Assumptions  \ref{ass:bounded2}, \ref{ass:bounded} and \ref{ass:lipschitz} are satisfied. Then 
\[  \mbe_{\hat D | D}[ || g_{\hat \bz } - g_\bz ||^2_{\cH_K } ] \leq c_\alpha L^2M^2\frac{\gamma^{2\alpha}}{N^\alpha} \;,\]
for some $c_\alpha <\infty$. 
\end{lemma}

\[\]

\begin{lemma}\cite[Eq. (38)]{fang2020optimal}
\label{lem:2}
Suppose Assumptions \ref{ass:bounded2}, \ref{ass:bounded} and \ref{ass:lipschitz} are satisfied. Then 
\[  \mbe_{\hat D | D}[ || T_{\hat \bx} - T_\bx ||^2 ] \leq c_\alpha \kappa^2 L^2 \frac{\gamma^{2\alpha}}{N^\alpha} \;,\]
for some $c_\alpha <\infty$. 
\end{lemma}

\[\]

\begin{lemma}\cite[Lemma 8]{fang2020optimal}
\label{lem:3}
Suppose Assumptions \ref{ass:bounded2}, \ref{ass:bounded} and \ref{ass:lipschitz} are satisfied and let $\lambda >0$. Define 
\begin{equation}
\label{eq:cxs}
 \cC_{\bx}(\lambda ):= \left(  \frac{\cA_\bx(\lambda)}{\lambda} + \frac{1}{\lambda^{\frac{3}{2}} N^{\frac{\alpha}{2}}} + \frac{1}{\sqrt \lambda} \right)  \;,
\end{equation} 
where 
\[  \cA_\bx(\lambda) := || (T_K + \lambda Id)^{-\frac{1}{2}}(T_K - T_\bx)||\;.  \] 
Then 
\[   \mbe_{\hat D | D}[ ||T_K^{\frac{1}{2}}( T_{\hat \bx} + \lambda Id )^{-1}||^2 ] \leq 
C^2_{\kappa, \gamma, L, \alpha}\cC_{\bx}(\lambda )^2  
  \;, \]
for some $C_{\kappa, \gamma, L, \alpha} <\infty$.
\end{lemma}


\section{Solving Distribution Regression with Tail-Averaged Gradient Descent}
\label{app:GD}

In this section we derive the learning properties of tail-averaged two-stage Gradient Descent. This is a necessary step for deriving our learning bounds for 
SGD on distribution regression problems and is of independent interest.

Let us begin with introducing the gradient updates using the second-stage data  $ D=\{(\mu_{ \hat x_j} , y_j)\}_{j=1}^n \subset \mu(\cM^+(\cX)) \times \cY$ 
as  $\hat f_0 = 0$ and for $t \geq 1$ 
\[ \hat f_{t+1} = \hat f_t - \eta \frac{1}{n}\sum_{j=1}^n ( \hat f_t(\mu_{\hat x_j})- y_j)K_{\mu_{\hat x_j}} \;. \]
Here, $\eta>0$ is the constant step-size. We furthermore set 
\begin{equation}
\label{eq:tail-ava-GD}
 \bar{ \hat f }_T  := \frac{2}{T}\sum_{t= \floor{T/2} +1}^T \hat f_t\;.  
\end{equation}

Similarly, we introduce the Gradient Descent updates using the first-stage data $ D=\{(\mu_{ x_j} , y_j)\}_{j=1}^n \subset \mu(\cM^+(\cX)) \times \cY$  
with initialization $f_0=0$ as 
\begin{equation*}
\label{eq:GD-first}
f_{t+1} = f_{t} -\eta \frac{1}{n}\sum_{j=1}^n ( f_t(\mu_{ x_j})- y_j)K_{\mu_{ x_j}}  
\end{equation*} 
and 
\begin{equation*}
 \bar f_T := \frac{2}{T}\sum_{t= \floor{T/2} +1}^T  f_t \;. 
\end{equation*}

We analyze the learning properties of \eqref{eq:tail-ava-GD}  based on the decomposition

\begin{equation*}
\bar{ \hat f }_T - f_\rho =  
\underbrace{(\bar{ \hat f }_T - \bar{ f }_T )}_{GD \;\; Variance \;\; 2. \;\;stage} +  \underbrace{ (\bar{ f }_T - f_\rho )}_{Error \;\; GD \;\; first \;\; stage} \;. 
\end{equation*}

The bound for the second stage GD variance is derived in Section \ref{sec:GDvariance-second}. The error estimate for first stage Gradient Descent 
is known from previous results. For completeness sake we review the main results in our setting in Section \ref{sec:GDfirst}


\subsection{A General Result}
\label{sec:general}

In this section we state a general result for spectral regularization algorithms. Those bounds are known for some time 
in learning theory. We collect some results 
from \cite{optimalratesRLS, fischer2017sobolev, blanchard2018optimal, lin2020optimal, mucke2019beating}.

We let  $ \{ g_\lam : [0, ||T_K||] \to [0, \infty)\;:\; \lam \in (0, ||T_K||] \}$  be a family of filter functions (for the definition we refer to 
one of the above mentioned papers). We define  
\[ \hat u_\lam  := g_\lam(T_\bx)S^*_\bx\by \;, \quad u_\lam:=g_\lam(T_K)S_K^*f_\rho \;. \]

Our aim is to provide a bound of the estimation error in different norms $||T_K^a( \hat u_\lambda   - u_\lam  )  ||_{\cH_K}$, for  $a \in [0,1/2]$.  
To this end, we require a condition for the observation noise. 

\vspace{0.3cm}

\begin{assumption}[Bernstein Observation Noise]
\label{ass:Bernstein}
For some $\sigma>0$, $B>0$ and for all $m \geq 2$ we have almost surely 
\[   \int_\cY |y - f_\rho (x)|^m \; \rho(dy |x ) \leq \frac{1}{2}m! \sigma^2 B^{m-2}\;. \]
\end{assumption}

\vspace{0.3cm}

Finally, we need 

\vspace{0.3cm}

\begin{assumption}
\label{ass:n-big}
Let $\lam >0$. Suppose that 
\[ n \geq \frac{32\kappa^2 \log(4/\delta)}{\lambda} \log\left(  e\cN(\lambda )\left(1 + \frac{\lambda }{||T_K||} \right)   \right) \;.  \]
\end{assumption}

\begin{proposition}[Estimation Error]
\label{prop:GD-prelim}
Let $a \in [0, 1/2]$. 
Let further $\delta \in (0,1]$ and suppose Assumptions \ref{ass:Bernstein}, \ref{ass:n-big} are satisfied. 
Denote 
\[  B_\lam = \max\{B, ||  S_Ku_\lam - f_\rho  ||_\infty\} \;. \]
\begin{enumerate}
\item Assume $f_\rho \in Range(L_K^\zeta)$, for some $\zeta \in (0,1]$. With probability not less than $1-\delta$, 
\begin{align*}
||T_K^a( \hat u_\lambda   - u_\lam  )  ||_{\cH_K}&\leq C_1 \log(12/\delta )\frac{ \lam^{a-1/2}}{\sqrt n}\;
   \left(  \sigma\sqrt{ \cN(\lam)} +  \frac{||S_K u_\lam- f_\rho||_{L^2}}{\sqrt \lam} + \frac{B_\lam}{\sqrt{n \lam }} \right) \\
   & +  C_2 \lam^{a+1/2}\; || T_\lam^{-1/2} u_\lam ||_{\cH_K} +  C_3 \lam^{a-1/2} ||S_K u_\lam - f_\rho ||_{L^2} \;,
\end{align*}
for some $C_1 >0$, $C_2 >0$ and $C_3 >0$.
\item Assume $f_\rho  \in Range(L_K^\zeta)$, for some $\zeta>1$. With probability not less than $1-\delta$, 
\begin{align*}
||T_K^a( \hat u_\lambda   - u_\lam  )  ||_{\cH_K}&\leq  C'_1\log(12/\delta )\frac{ \lam^{a-1/2}}{\sqrt n}\;
   \left(  \sigma\sqrt{ \cN(\lam)} +  \frac{||S_K u_\lam- f_\rho||_{L^2}}{\sqrt \lam} + \frac{B_\lam}{\sqrt{n \lam }} \right) \\
   & +  C'_2 \lam^{a}  \; || T_K^{-\zeta } u_\lam ||_{\cH_K}  \left(  \frac{\log(4/\delta)}{\sqrt n}   + \lam^{\zeta }\right) 
   +  C'_3 \lam^{a-1/2} ||S_K u_\lam - f_\rho ||_{L^2} \;,
\end{align*}
for some $C'_1 >0$, $C'_2 >0$ and $C'_3 >0$.
\end{enumerate}
\end{proposition}

\[\]

\begin{proof}[Proof of Proposition \ref{prop:GD-prelim}]
Let $a \in [0, 1/2]$. We write 
\begin{align*}
& T_K^a( \hat u_\lambda   - u_\lam  )  \\ 
&= \underbrace{T_K^a g_\lam(T_\bx) ( (S^*_\bx\by - S_K^*f_\rho ) -( T_\bx u_\lam - T_K u_\lam )) }_{\cT_1} + 
\underbrace{T_K^a(g_\lam(T_\bx)T_\bx - Id)u_\lam }_{\cT_2} + \underbrace{T_K^a g_\lam(T_\bx)(S_K^*f_\rho - T_K u_\lam )}_{\cT_3}  \;. 
\end{align*}
We further set $T_{\bx , \lam }:= T_\bx + \lam $ and $T_{\lam }:= T_K + \lam $.  

{\bf Bounding $\cT_1$.} 
We further decompose 
\begin{align*}
 ||\cT_1 ||_{\cH_K}  &\leq ||T_K^a T_\lam^{-1/2}|| \cdot  || T_\lam^{1/2} T_{\bx , \lam }^{-1/2}|| \cdot  ||T_{\bx , \lam }^{1/2} g_\lam(T_\bx)T_{\bx , \lam }^{1/2}||  
  \cdot  ||T_{\bx , \lam }^{-1/2}(S^*_\bx\by - S_K^*f_\rho) -( T_\bx u_\lam - T_K u_\lam ))||_{\cH_K} \;. 
\end{align*}
A short calculation shows that 
\[ ||T_K^a T_\lam^{-1/2}|| \leq \lam^{a-1/2} \;.  \]
By \cite{optimalratesRLS}, Proof of Theorem 4, with probability at least $1-\delta/6$ 
\[  || T_\lam^{1/2} T_{\bx , \lam }^{-1/2}|| \leq \sqrt 2 \;,  \]
provided Assumption \ref{ass:n-big} is satisfied. Moreover, according to \cite{bauer2007regularization}, Definition 1 we have almost surely 
\[  ||T_{\bx , \lam }^{1/2} g_\lam(T_\bx)T_{\bx , \lam }^{1/2}||  \leq E\;, \] 
for some $E>0$. Finally,  Lemma 6.10 in \cite{fischer2017sobolev}  shows that with probability at least $1-\delta/6$ 
\begin{align*}
 ||T_{\bx , \lam }^{-1/2}((S^*_\bx\by - S_K^*f_\rho ) -( T_\bx u_\lam - T_K u_\lam ) )||^2_{\cH_K} 
&\leq C_\kappa \log^2(12/\delta )\frac{1}{n}\; \left(  \sigma^2 \cN(T_K) +  \frac{||S_K u_\lam- f_\rho||^2_{L^2}}{\lam} + \frac{B^2_\lam}{n \lam} \right) \;,
\end{align*}
for some $C_\kappa>0$. Thus, combining the above gives us  with probability at least $1-\delta/3$ 
\begin{equation}
\label{eq:t1}
 ||\cT_1 ||_{\cH_K}  \leq C'_\kappa \log(12/\delta )\frac{ \lam^{a-1/2}}{\sqrt n}\;
   \left(  \sigma\sqrt{ \cN(T_K)} +  \frac{||S_K u_\lam- f_\rho||_{L^2}}{\sqrt \lam} + \frac{B_\lam}{\sqrt{n \lam }} \right) \;,
\end{equation}
for some $C'_\kappa>0$.

{\bf Bounding $\cT_2$.} 
Setting $ r_\lam(T_\bx) = g_\lam(T_\bx)T_\bx - Id$, we split once more and obtain with probability at least $1-\delta$  
\begin{align*}
 ||\cT_2 ||_{\cH_K}  &\leq  ||T_K^a T_\lam^{-1/2}|| \cdot  || T_\lam^{1/2} T_{\bx , \lam }^{-1/2}|| \cdot  ||T_{\bx , \lam }^{1/2} r_\lam(T_\bx) u_\lam || \\
 &\leq \sqrt 2 \lam^{a-1/2} ||T_{\bx , \lam }^{1/2} r_\lam(T_\bx) u_\lam ||_{\cH_K} \;. 
\end{align*}
Now we follow the proof of \cite{mucke2019beating}, Proposition 2, slightly adapted to our purposes, and consider two different cases: 

{\bf (a)} Assume $f_\rho \in Range(L_K^\zeta)$, for some $\zeta \in (0,1]$. Here, we write  with probability at least $1-\delta/3$  
\begin{align*}
||T_{\bx , \lam }^{1/2} r_\lam(T_\bx) u_\lam ||_{\cH_K} &\leq ||T_{\bx , \lam }r_\lam(T_\bx)|| \cdot  
      || T_{\bx , \lam }^{-1/2} T_\lam^{1/2} ||  \cdot || T_\lam^{-1/2} u_\lam ||_{\cH_K} \\
&\leq C_1 \lam   \; || T_\lam^{-1/2} u_\lam ||_{\cH_K}\;,
\end{align*}
for some $C_1>0$. Thus, 
\[  ||\cT_2 ||_{\cH_K}  \leq  C_1 \lam^{a+1/2}\; || T_\lam^{-1/2} u_\lam ||_{\cH_K} \;.  \]

{\bf (b)}  Assume $f_\rho  \in Range(L_K^\zeta)$, for some $\zeta>1$. In this case we let  $\zeta \geq 1$ and have for some $C_2>0$
\begin{align*}
||T_{\bx , \lam }^{1/2} r_\lam(T_\bx) u_\lam ||_{\cH_K}  &\leq || T_{\bx , \lam }^{1/2} r_\lam(T_\bx)|| \; ||T_K^\zeta -T_\bx^\zeta || \;  
    || T_K^{-\zeta } u_\lam ||_{\cH_K} + ||  T_{\bx , \lam }^{1/2} r_\lam(T_\bx) T_\bx^\zeta || \; || T_K^{-\zeta } u_\lam ||_{\cH_K} \\
    &\leq  \sqrt \lam \; || T_K^{-\zeta } u_\lam ||_{\cH_K} \left( C_2 \; \frac{\log(4/\delta)}{\sqrt n}   + \lam^{\zeta }\right) \;,
\end{align*}
holding with probability at least $1-\delta/3$. Thus, for some $\tilde C_2 >0$.
\[  ||\cT_2 ||_{\cH_K}  \leq \tilde C_2 \lam^{a}  \; || T_K^{-\zeta } u_\lam ||_{\cH_K}  \left(  \frac{\log(4/\delta)}{\sqrt n}   + \lam^{\zeta }\right) \;.  \]

{\bf Bounding $\cT_3$.} Applying standard arguments gives  with probability at least $1-\delta/3$
\begin{align*}
 ||\cT_3 ||_{\cH_K}  &\leq ||T_K^a g_\lam(T_\bx)(S_K^*f_\rho - T_K u_\lam )||_{\cH_K} \\
 &\leq C_3 \lam^{a-1/2} ||S_K u_\lam - f_\rho ||_{L^2}\;.
\end{align*}
Combining all of our findings leads to result.  
\end{proof}

\vspace{0.3cm}

We summarize some results under refined assumptions on $f_\rho$, see e.g. \cite{fischer2017sobolev}, Lemma 6.6. and Corollary 6.7\;.  

\begin{lemma}
\label{lem:some-bounds}
Suppose Assumptions \ref{ass:bounded} and \ref{ass:source} are satisfied. Then
\begin{enumerate}
\item $||u_\lam||_{\cH_K } \leq  R\lam^{r-\frac{1}{2}}$, 
\item $|| S_K u_\lam - f_\rho  ||_{L^2} \leq R \lam^{r}$, 
\item $||S_K u_\lam||_\infty \leq \kappa^2 R \lam^{-|1/2-r|_+}$, 
\item $||(T_K + \lam)^{-1/2} u_\lam||_{\cH_K} \leq CR \lam^{r-1}$, 
\item Assume $||f_\rho||_\infty < \infty$. Then $|| S_K u_\lam - f_\rho  ||_\infty \leq (||f_\rho||_\infty + \kappa^2 R)\lam^{-|1/2-r|_+}$. 
\end{enumerate}
\end{lemma}


\subsection{Bounding First Stage Tail-Averaged Gradient Descent Error}
\label{sec:GDfirst}

Our aim is now to state an error bound for the first-stage tail-average GD algorithm, defined in \eqref{eq:GD-first}. According to the results 
in \cite{mucke2019beating},  \eqref{eq:GD-first} can be rewritten as 
\[  \bar f_T = G_T(T_\bx) S_\bx^* \by \;, \]
where $G_T$ is defined in \eqref{eq:GD-first}, Appendix B and constitutes a family of filter functions $\{G_\lam\}_\lam$, where we set $\lam = 1/(\eta T)$. 
We then get 
\[  || S_K\bar f_T - f_\rho ||_{L^2 } \leq   || S_K(\bar f_T - \bar u_T)  ||_{L^2 } + || S_K\bar u_T - f_\rho ||_{L^2 }   \;, \]
where we set $ \bar u_T = G_T(T_K) S_K^* f_\rho$. Thus, the first term corresponds the an estimation error which we bound by means of 
Proposition \ref{prop:GD-prelim} with $a=1/2$. Note that Assumption \ref{ass:Bernstein} is satisfied with $\sigma = B = 2M $, since $\cY \subseteq [-M,M]$  
by assumption.

\vspace{0.5 cm}

\begin{proposition}[Excess Risk Tail-Ave GD First Stage]
\label{prop:GD-usual}
Suppose Assumptions \ref{ass:bounded},  \ref{ass:n-big} are satisfied. Let $T \in \mbn$ and denote 
\[  B_T= \max\{2M, || S_K\bar u_T - f_\rho  ||_\infty\} \;. \]
Let further $\delta \in (0,1]$, $\lambda = (\eta T)^{-1}$. 
\begin{enumerate}
\item Assume $f_\rho \in Range(L_K^\zeta)$, for some $\zeta \in (0,1]$. With probability not less than $1-\delta$, 
\begin{align*}
|| S_K\bar f_T - f_\rho ||_{L^2 }   &\leq C_1 \log(12/\delta )\frac{1}{\sqrt n}\;
   \left(  M \sqrt{ \cN(\lam)} +  \frac{||S_K \bar u_T - f_\rho||_{L^2}}{\sqrt \lam} + \frac{B_T}{\sqrt{n \lam }} \right) \\
   & +  C_2 \lam\; || (T_K + \lam)^{-1/2} \bar u_T ||_{\cH_K} +  C_3  ||S_K \bar u_T- f_\rho ||_{L^2} \;,
\end{align*}
for some $C_1 >0$, $C_2 >0$ and $C_3 >0$.
\item Assume $f_\rho  \in Range(L_K^\zeta)$, for some $\zeta>1$. With probability not less than $1-\delta$, 
\begin{align*}
|| S_K\bar f_T - f_\rho ||_{L^2 }   &\leq  C'_1\log(12/\delta )\frac{ 1}{\sqrt n}\;
   \left(   M  \sqrt{ \cN(\lam)} +  \frac{||S_K \bar u_T - f_\rho||_{L^2}}{\sqrt \lam} + \frac{B_T}{\sqrt{n \lam }} \right) \\
   & +  C'_2 \lam^{1/2}  \; || T_K^{-\zeta } \bar u_T ||_{\cH_K}  \left(  \frac{\log(4/\delta)}{\sqrt n}   + \lam^{\zeta }\right) 
   +  C'_3 ||S_K \bar u_T - f_\rho ||_{L^2} \;,
\end{align*}
for some $C'_1 >0$, $C'_2 >0$ and $C'_3 >0$.
\end{enumerate}
\end{proposition}

\[\]

\begin{corollary}[Rate of Convergence First Stage Tail-Averaged Gradient Descent]
\label{cor:rates-first-stage-GD}
Suppose all Assumptions of Proposition \ref{prop:GD-usual} are satisfied. Let additionally Assumptions \ref{ass:source} and \ref{ass:eigenvalue}  hold.  
Then with probability not less than $1-\delta$, 
the excess risk for the first stage tail-averaged Gradient Descent satisfies with probability not less than $1-\delta$:
\begin{enumerate}
\item If $2r+\nu >1$: Let $\eta_n T_n = \left( \frac{R^2}{M^2}n\right)^{\frac{1}{2r + \nu}}$, then 
\[  || S_K\bar f_T - f_\rho ||_{L^2 } \leq C \; \log (12/\delta ) R \left(\frac{M^2}{R^2n} \right)^{\frac{r}{2r + \nu}} \;.\]
for some constant $C < \infty$. 
\item  If $2r+\nu \leq 1$: Let $\eta_n T_n = \frac{R^2 n}{M^2 \log^K(n)}$ for some $K>1$, then  
\[  || S_K\bar f_T - f_\rho ||_{L^2 } \leq C' \; \log (12/\delta ) \left( \frac{M^2 \log^K(n)}{R^2n} \right)^r \;.\]
for some constant $C' < \infty$.
\end{enumerate}
\end{corollary}

\begin{proof}[Proof of Corollary \ref{cor:rates-first-stage-GD}]
The proof follows from standard calculations by applying Lemma \ref{lem:some-bounds} and Proposition \ref{prop:GD-usual} with $\zeta = r$. 
\end{proof}


\subsection{Bounding Second Stage Tail-Averaged Gradient Descent Variance}
\label{sec:GDvariance-second}

Based on the notation introduced in Section \ref{sec:notation}, the GD updates can be rewritten as 
\begin{equation}
\label{eq:GD-update-second}
  \hat f_{t+1} = \hat f_t - \eta (T_{\hat \bx}\hat f_t - g_{\hat \bz}) 
\end{equation}  
and 
\[ f_{t+1} =  f_t - \eta (T_{ \bx} f_t - g_{ \bz}) \;. \]
We thus find for any $t \geq 1$
\begin{equation*}
\hat f_{t+1} - f_{t+1} = (Id - \eta T_{\hat \bx} )(\hat f_{t}-f_{t}) + \eta \hat \xi_{t} \;,
\end{equation*}
where we define the noise variables as 
\begin{equation}
\label{eq:noise}
 \hat \xi_t := \hat \xi^{(1)}_{t} + \hat \xi^{(2)} \; \quad \hat \xi^{(1)}_{t} := (T_{\bx} - T_{\hat \bx})f_t \;, \quad \hat \xi^{(2)} := g_{\hat \bz } - g_\bz \;. 
\end{equation}  
By induction we easily find 
\begin{equation}
\label{eq:diffx}
\hat f_{t+1} - f_{t+1} = \eta\sum_{s=0}^t (Id - \eta T_{\hat \bx} )^{t-s} \hat \xi_s 
\end{equation}
and 
\begin{equation*}
\label{eq:mean}
\bar{ \hat f }_T - \bar f_T   = \frac{2\eta }{T} \sum_{t= \floor{T/2} +1}^T \sum_{s=0}^{t-1} (Id - \eta T_{\hat \bx} )^{t-1-s} \hat \xi_s \;.
\end{equation*}

\vspace{0.3cm}

As a first step we need to bound the norm of the noise variables \eqref{eq:noise}. To this end, let us introduce the GD updates $u_t=g_t(T_K)S_K^*f_\rho$, 
where $\{g_t\}_t$ is a filter function, given in Eq. $(23)$ in  \cite{mucke2019beating}.

\vspace{0.3cm}

\begin{proposition}
\label{prop:norm-bound}
Let further $\delta \in (0,1]$ and suppose Assumptions \ref{ass:Bernstein}, \ref{ass:n-big} are satisfied. 
Denote 
\[  B_t = \max\{M, || S_K u_t -  f_\rho ||_\infty\} \;. \]
\begin{enumerate}
\item Assume $f_\rho \in Range(L_K^\zeta)$, for some $\zeta \in (0,1]$. With probability not less than $1-\delta$, 
\begin{align*}
|| f_t   ||_{\cH_K}&\leq C_1 \log(12/\delta )\sqrt{\frac{ \eta t}{ n}}\;
   \left(  M\sqrt{ \cN(1/(\eta t))} +  \sqrt{\eta t} ||S_K u_t- f_\rho||_{L^2}+ B_t \sqrt{\frac{ \eta t}{ n}}  \right) \\
   & +  C_2 (\eta t)^{-1/2}\; || (T_K +1/(\eta t))^{-1/2} u_t ||_{\cH_K} +  C_3 \sqrt{\eta t} ||S_K u_t - f_\rho ||_{L^2} +  ||u_t||_{\cH_K} \;,
\end{align*}
for some $C_1 >0$, $C_2 >0$ and $C_3 >0$.
\item Assume $f_\rho  \in Range(L_K^\zeta)$, for some $\zeta\geq 1$. With probability not less than $1-\delta$, 
\begin{align*}
||f_t    ||_{\cH_K}&\leq  C'_1\log(12/\delta )\sqrt{\frac{ \eta t}{ n}}\;
   \left(  M\sqrt{ \cN(1/(\eta t))} +  \sqrt{\eta t}||S_K u_t- f_\rho||_{L^2} + B_t \sqrt{\frac{ \eta t}{ n}}  \right) \\
   & +  C'_2   \; || T_K^{-\zeta } u_t ||_{\cH_K}  \left(  \frac{\log(4/\delta)}{\sqrt n}   + (\eta t)^{-\zeta }\right) 
   +  C'_3 \sqrt{\eta t}||S_K u_t - f_\rho ||_{L^2} + ||u_t||_{\cH_K} \;,
\end{align*}
for some $C'_1 >0$, $C'_2 >0$ and $C'_3 >0$.
\end{enumerate}
\end{proposition}

\begin{proof}[Proof of Proposition \ref{prop:norm-bound}]
We decompose as 
\[ ||f_t  ||_{\cH_K} \leq ||f_t  - u_t    ||_{\cH_K}  + ||u_t||_{\cH_K} \;. \]
The proof follows then from Proposition \ref{prop:GD-usual} with $a=0$, $\sigma=B=M$. 
\end{proof}

\vspace{0.5cm}

\begin{corollary}
\label{prop:GD-bound2}
In addition to all assumptions of Proposition \ref{prop:norm-bound}, 
suppose Assumptions \ref{ass:bounded}, \ref{ass:source} and \ref{ass:eigenvalue} are satisfied. 
\begin{enumerate}
\item Let $0<r\leq 1/2$ and assume that  
\begin{equation}
\label{eq:eff-mis-wellx}
  n \geq   64 \kappa^2 \log(12/\delta) (\eta t) \log((\eta t)^\nu) \;. 
\end{equation}
Then Assumption \ref{ass:n-big} is satisfied and with probability not less than $1-\delta$
\[ || f_t   ||_{\cH_K} \leq C'_{\kappa, M, R}   (\eta t)^{\frac{1}{2}\max \{ \nu, 1-2r \} } \;,  \]
for some $ C'_{\kappa, M, R} >0$. 
\item Let $1/2 \leq r \leq 1$ and assume that 
\begin{equation}
\label{eq:eff-dim-wellx}
  n \geq   64 e \kappa^2 \log^2(12/\delta) (\eta t)^{1+\nu} \;. 
\end{equation}  
Then Assumption \ref{ass:n-big} is satisfied and 
\[ || f_t   ||_{\cH_K} \leq C_{\kappa, M, R}\;, \] 
with probability not less than $1-\delta$ and for some $ C_{\kappa, M, R} >0$.
\item Let $1 <r $ and assume that \eqref{eq:eff-dim-wellx} is satisfied. Then 
\[ || f_t   ||_{\cH_K} \leq C'_{\kappa, M, R}  \;,  \]
with probability not less than $1-\delta$ and for some $ C'_{\kappa, M, R} >0$.
\end{enumerate}
\end{corollary}

\begin{proof}[Proof of Corollary \ref{prop:GD-bound2}]
Recall that by Lemma \ref{lem:some-bounds} we have 
\[  || S_K u_t - f_\rho  ||_{L^2} \leq R (\eta t)^{-r} \,, \quad || S_K u_t - f_\rho  ||_\infty \leq (M + \kappa^2 R) (\eta t)^{|1/2-r|_+}  \]
and 
\[ ||u_t||_{\cH_K } \leq  R(\eta t)^{\frac{1}{2}-r} \;, \quad ||(T_K + 1/(\eta t))^{-1/2} u_t|| \leq CR (\eta t)^{1-r} \;.  \] 
\begin{enumerate}
\item
Now suppose that $0<r\leq 1/2$. 
The first part of Proposition \ref{prop:norm-bound} yields with probability not less than $1-\delta$ 
\begin{align}
\label{eq:cor1}
|| f_t   ||_{\cH_K} &\leq C_1 \log(12/\delta )\sqrt{\frac{ \eta t}{ n}}\;
   \left(  M (\eta t)^{\nu/2}  +  R (\eta t)^{1/2-r} +(\eta t)^{|1/2-r|_+} \sqrt{\frac{ \eta t}{ n}}  \right) +  C_2 \; R (\eta t)^{1/2-r} \;.
\end{align}
Then \eqref{eq:cor1} and \eqref{eq:eff-mis-wellx} give with $\log(12/\delta) \geq 1$ the bound 
\begin{align*}
|| f_t   ||_{\cH_K} &\leq C_{\kappa, M, R} \log(12/\delta )  \sqrt{\frac{ \eta t}{ n}}  (\eta t)^{\frac{1}{2}\max\{\nu, 1-2r\}} 
 +  C'_{\kappa, M, R} (\eta t)^{1/2-r} \\
 &\leq  C''_{\kappa, M, R}  (\eta t)^{\frac{1}{2}\max \{ \nu, 1-2r \} } \;, 
\end{align*}
with probability not less than $1-\delta$. 

\item
Specifically, if $1/2 \leq r \leq 1$ and $1/(\eta t )\leq \kappa^2$, we have 
\begin{align*}
|| f_t   ||_{\cH_K} &\leq  C'_{\kappa, r} \max\{M, R\}\log(12/\delta )\;
   \left(   \frac{ 1}{\sqrt n}(\eta t)^{1/2 +\nu/2}  + 1+ \frac{ \eta t}{ n}  \right) +  C'_{\kappa, r} \; R \\
&\leq  C_{\kappa, M, R} \;,
\end{align*}
for some $C_{\kappa, M, R} >0$, provided \eqref{eq:eff-dim-wellx} is satisfied. 

\item
The second part of Proposition \ref{prop:norm-bound} then gives with $\zeta=r$ and $||T_K^{-r}u_t||_{\cH_K} \leq C R (\eta t)^{1/2}$, 
with probability not less than $1-\delta$ 
\begin{align*}
|| f_t   ||_{\cH_K} &\leq C_{1} \log(12/\delta )\;
   \left(   \frac{ 1}{\sqrt n}(\eta t)^{1/2 +\nu/2}  + 1+ \frac{ \eta t}{ n}  \right) + 
 C_2 (\eta t)^{1/2}\left( \frac{\log(4/\delta)}{\sqrt n} + (\eta t)^{-r}\right)  + C_3(\eta t)^{1/2-r}\\
 &\leq C_4 +  C_2 (\eta t)^{1/2}\left( \frac{\log(4/\delta)}{\sqrt n} + (\eta t)^{-r}\right) \\
 &\leq C_5 +  C_2 \log(4/\delta) \frac{(\eta t)^{1/2}}{\sqrt n}\\
 &\leq C_6 \;,
\end{align*}
for some $C_6>0$, depending on $\kappa, M, R$.
\end{enumerate}
\end{proof}

\vspace{0.5cm}

We now come to the main result of this subsection.

\vspace{0.3cm}

\begin{proposition}[Second Stage GD Variance]
\label{prop:GD-variance}
Suppose Assumptions \ref{ass:bounded}, \ref{ass:lipschitz} and \ref{ass:n-big} are satisfied. Let $\eta < 1/\kappa^2$, $T\geq 3$ and define 
\begin{equation}
\label{eq:B}
\cB(1/\eta T ):=  \left(  \frac{2 \eta T}{n  } +  \sqrt{\frac{\eta T\cN(1/\eta T)}{n }}    + 
\frac{\eta T}{ N^{\frac{\alpha}{2}}} + 1 \right) \;. 
\end{equation}
\begin{enumerate}
\item {\bf If $f_\rho \in Ran(S_K)$: } 
The GD variance satisfies with probability not less than $1-\delta$ with respect to the data $D$ 
\[ \mbe_{\hat D | D}[  ||S_K ( \bar{ \hat f }_T - \bar f_T  ) ||_{L^2}  ] \leq C \log(4/\delta)
\log (T) \frac{\sqrt{\eta T}}{N^{\frac{\alpha}{2}}} \cB(1/ \eta T )  \;, \]
for some $C < \infty$, depending on $\kappa, \gamma, L, \alpha$.

\item {\bf If $f_\rho \not \in Ran(S_K)$: } Let us define 
\begin{align}
\label{eq:small-phi}
\varphi (\eta s) &= (\eta s)^{\frac{1}{2}\max\{ \nu , 1-2r \}} \;.
\end{align} 
With probability not less than $1-\delta$ with respect to  the data $D$ 
\begin{align*}
 \mbe_{\hat D | D}[ ||S_K ( \bar{ \hat f }_T - \bar f_T  ) ||_{L^2} ] &\leq   C \log(4/\delta) \log(T) \frac{ \sqrt{\eta T}\cB(1/\eta T) }{ N^{\frac{\alpha}{2}} } 
\left( 1  +  \varphi(\eta T)   \right)   \;,
\end{align*}
for some $C < \infty$, depending on $\kappa, \gamma, L, \alpha, r$.  
\end{enumerate}
\end{proposition}

\[\]

\begin{proof}[Proof of Proposition \ref{prop:GD-variance}]
By \eqref{eq:mean} we may write 
\begin{align*}
T^{\frac{1}{2}}_K ( \bar{ \hat f }_T - \bar f_T  ) &=  
\frac{2\eta }{T} \sum_{t= \floor{T/2} +1}^T \sum_{s=0}^{t-1} T^{\frac{1}{2}}_K(Id - \eta T_{\hat \bx} )^{t-1-s} \hat \xi_s \\
&= \frac{2\eta }{T} \sum_{t= \floor{T/2} +1}^T \sum_{s=0}^{t-1} A\cdot B_{t,s} \;  \hat \xi_s\;,
\end{align*}
where for $\lambda >0$ we introduce 
\begin{align*}
 A &:= T^{\frac{1}{2}}_K ( T_{\hat \bx} + \lambda Id )^{-1} \\
B_{t,s} &:= ( T_{\hat \bx} + \lambda Id ) (Id - \eta T_{\hat \bx} )^{t-1-s}  \;.
\end{align*}

{\bf (1) Bounding the operator $A$:} This follows directly from Lemma \ref{lem:3}. Indeed, 
\[  \mbe_{\hat D | D}[ ||T^{\frac{1}{2}}_K ( T_{\hat \bx} + \lambda Id )^{-1}|| ]  \leq C_{\kappa, \gamma, L, \alpha}
\left(  \frac{\cA_\bx(\lambda)}{\lambda} + \frac{1}{\lambda^{\frac{3}{2}} N^{\frac{\alpha}{2}}} + \frac{1}{\sqrt \lambda} \right)  \;,\]
where by \cite[Proposition 5.3]{blanchard2018optimal} the term $\cA_\bx(\lambda)$ satisfies with probability at least $1-\delta$ with respect to the data $D$ 
\begin{align}
\label{eq:Alambda}
\cA_\bx(\lambda )&=  || (T_K + \lambda Id)^{-\frac{1}{2}}(T_K - T_\bx )|| \\
&\leq 2\log(2/\delta) \left( \frac{2}{n \sqrt{\lambda}} + \sqrt{\frac{\cN(\lambda)}{n}} \right)\;. 
\end{align}
Hence, since $1\leq  2\log(2/\delta)$ for any $\delta \in (0,1]$ we obtain with probability at least $1-\delta$
\begin{equation}
\label{eq:A}
 \mbe_{\hat D | D}[ ||T^{\frac{1}{2}}_K ( T_{\hat \bx} + \lambda Id )^{-1}|| ]  \leq C_{\kappa, \gamma, L, \alpha}\log(2/\delta)\frac{\cB(\lambda ) }{\sqrt \lambda} \;.  
\end{equation}

{\bf (2) Bounding the operators $B_{t,s}$:}
We write 
\[  ||B_{t,s}|| \leq  ||T_{\hat \bx} (Id - \eta T_{\hat \bx} )^{t-1-s}|| + \lambda ||(Id - \eta T_{\hat \bx} )^{t-1-s}||  \;. \]
Denoting $\sigma_1 \geq \sigma_2 \geq ... $ the sequence of eigenvalues of $T_{\hat \bx}$, we have  for any $s = 0, ... ,t-1$, $t=\floor{\frac{T}{2}}, ..., T$ 
the upper bound 
\begin{align*}
||(Id - \eta T_{\hat \bx} )^{t-1-s}|| &\leq \sup_{j} |(1-\eta \sigma_j )^{t-1-s}| \leq 1 \;,
\end{align*}
since $\eta < 1/\kappa^2$.

For bounding the first term note that for $s=t-1$ we have 
\begin{equation*}
 ||T_{\hat \bx} (Id - \eta T_{\hat \bx} )^{t-1-s}|| =  ||T_{\hat \bx} || \leq \kappa^2 \;. 
\end{equation*}
If $0\leq s < t-1$ we use the inequality $1+\sigma \leq e^\sigma$ for any $\sigma \geq -1$. Then a short calculation 
gives\footnote{The function $h(\sigma)=\sigma e^{-c\sigma}$, $c>0$, achieves it's maximum at $\sigma=1/c$.} 
\begin{align*}
 ||T_{\hat \bx} (Id - \eta T_{\hat \bx} )^{t-1-s}|| &\leq \sup_{j}|\sigma_j (1-\eta \sigma_j )^{t-1-s}| \\
&\leq \sup_{j} \sigma_j e^{-\eta (t-1-s) \sigma_j }  \\
&\leq \frac{1}{e \eta (t-1-s)}\;.
\end{align*}
Thus, combining the above findings yields 
\begin{equation}
\label{eq:Bt-1}
 ||B_{t, t-1}|| \leq  \kappa^2 + \lambda \;,  
\end{equation}
and for $0\leq s < t-1$ 
\begin{equation}
\label{eq:Bt-2}
 ||B_{t,s}|| \leq  \frac{1}{e \eta (t-1-s)} + \lambda  \;. 
\end{equation}

{\bf (3) Bounding the noise variables $\hat \xi_s$:}
Recall that $\hat \xi_s := \hat \xi^{(1)}_{s} + \hat \xi^{(2)} $ with 
\[  \hat \xi^{(1)}_{s} := (T_{\bx} - T_{\hat \bx})f_s \;, \qquad \hat \xi^{(2)} := g_{\hat \bz } - g_\bz \;. \]
Applying Lemma \ref{lem:1}  gives for some $c_\alpha < \infty$ the bound 
\[   \mbe_{\hat D | D}[ ||\xi^{(2)} ||_{\cH_K} ] =  \mbe_{\hat D | D}[ ||g_{\hat \bz } - g_\bz ||_{\cH_K} ] 
\leq c_\alpha LM\frac{\gamma^{\alpha}}{N^{\frac{\alpha}{2}}} \;.  \]
\begin{itemize}
\item {\bf $f_\rho \in Ran(S_K)$: }
Lemma \ref{lem:2} and Proposition \ref{prop:GD-variance} gives  with probability at least $1-\delta$
\[   \mbe_{\hat D | D}[ ||\xi_s^{(1)} ||_{\cH_K} ] = 
\mbe_{\hat D | D}[ || (T_{\hat \bx} - T_\bx)f_s ||_{\cH_K } ] \leq  \frac{C}{N^{\frac{\alpha}{2}}} \;, \]
with $C= c'_\alpha \gamma^{\alpha} \kappa L(2 ||w_\rho||_{\cH_K}  + 1)$, for some $c_\alpha' < \infty$. 
Combining both bounds finally leads to 
\begin{align}
\label{eq:xi1}
  \mbe_{\hat D | D}[ ||\hat \xi_s ||_{\cH_K} ] &\leq \frac{c''_\alpha}{N^{\frac{\alpha}{2}}} \;,
\end{align}
where $c''_\alpha = 2\max\{C, c_\alpha \gamma^{\alpha} LM \}$ and holding with probability at least $1-\delta$.

\item {\bf $f_\rho \not \in Ran(S_K)$: } In this case we apply Lemma \ref{lem:2} and Corollary \ref{prop:GD-bound2} to get  with 
probability at least $1-\delta$ with respect to the data $D$ 
\[   \mbe_{\hat D | D}[ ||\xi_s^{(1)} ||_{\cH_K} ] \leq \kappa  c'_\alpha \gamma^{\alpha} L \log(6/\delta) \frac{\varphi (\eta s) }{N^{\frac{\alpha}{2}}} \;, \]
for some $c'_\alpha  < \infty$ and where  $ \varphi$ is defined in \eqref{eq:small-phi}. 
Hence, with probability at least $1-\delta$ with respect to the data $D$ one has 
\begin{equation}
\label{eq:xi2}
  \mbe_{\hat D | D}[ ||\hat \xi_s ||_{\cH_K} ] \leq  \frac{ \tilde c''_\alpha }{ N^{\frac{\alpha}{2}} } ( 1 + \varphi(\eta s)  ) \;,  
\end{equation}  
for some $c''_\alpha  < \infty$. 
\end{itemize}

We complete the proof now by collecting the above findings. Let us write 
\begin{align*}
|| T^{\frac{1}{2}}_K ( \bar{ \hat f }_T - \bar f_T  ) ||_{\cH_K} &\leq 
\underbrace{ \frac{2\eta }{T} \sum_{t= \floor{T/2} +1}^T \sum_{s=0}^{t-2} || A || \cdot ||B_{t,s} || \cdot || \hat \xi_s ||_{\cH_K} }_{\cI_1}+ 
\underbrace{ \frac{2\eta }{T} \sum_{t= \floor{T/2} +1}^T ||A|| \cdot ||B_{t,t-1}|| \cdot || \hat \xi_{t-1} ||_{\cH_K} }_{\cI_2}  \;. 
\end{align*}
We again distinguish between the two cases: 
\begin{itemize}
\item {\bf $f_\rho \in Ran(S_K)$: } From \eqref{eq:A}, \eqref{eq:Bt-1} and  \eqref{eq:xi1} we obtain with $\lambda \leq \kappa^2$
\begin{align}
\label{eq:I2}
 \mbe_{\hat D | D}[  \cI_2  ] &\leq \eta \tilde C_{\kappa, \gamma, L, \alpha}\log(2/\delta) \frac{\cB(\lambda )}{\lam^{\frac{1}{2}} N^{\frac{\alpha}{2}}} \;,
\end{align}
for some $ \tilde C_{\kappa, \gamma, L, \alpha} < \infty$. 
Additionally, by 
 \eqref{eq:A}, \eqref{eq:Bt-2}, \eqref{eq:xi1} and Lemma \ref{lem:phi-bounds} we find with $\lambda = (\eta T)^{-1}$
\begin{align}
\label{eq:I1}
 \mbe_{\hat D | D}[  \cI_1  ] &\leq 
  \tilde C'_{\kappa, \gamma, L, \alpha}\log(2/\delta) \frac{\cB(\lambda )}{\lam^{\frac{1}{2}}N^{\frac{\alpha}{2}}}
\frac{2\eta }{T} \sum_{t= \floor{T/2} +1}^T  \sum_{s=0}^{t-2}  \left( \frac{1}{e \eta (t-1-s)} + \lambda \right)  \nonumber \\
&\leq 4 \tilde C'_{\kappa, \gamma, L, \alpha}\log(2/\delta) \log(T) \frac{\sqrt{\eta T} \cB(1/\eta T)}{N^{\frac{\alpha}{2}}} \;,
\end{align} 
for some $ \tilde C'_{\kappa, \gamma, L, \alpha} < \infty$.

Combining \eqref{eq:I2} and \eqref{eq:I1} gives with $\lam=(\eta T)^{-1}$ 
\begin{align*}
 \mbe_{\hat D | D}[ ||S_K ( \bar{ \hat f }_T - \bar f_T  ) ||_{L^2} ] 
&\leq C''_{\kappa, \gamma, L, \alpha}\log(2/\delta) \log(T) \frac{\sqrt{\eta T}\cB((\eta T)^{-1} )}{N^{\frac{\alpha}{2}}}   \;,
\end{align*}
with probability at least $1-\delta$, for some $C''_{\kappa, \gamma, L, \alpha} < \infty$.

\[\]

\item {\bf $f_\rho \not \in Ran(S_K)$: } From \eqref{eq:A}, \eqref{eq:Bt-1}, \eqref{eq:xi2} we find 
\begin{align*}
 \mbe_{\hat D | D}[  \cI_2  ] &\leq   
\tilde  C_{\kappa, \gamma, L, \alpha}\log(2/\delta)  \frac{ \cB(\lambda) }{ \sqrt{\lambda} N^{\frac{\alpha}{2}} } 
\frac{2\eta }{T} \sum_{t= \floor{T/2} +1}^T ( 1 + \varphi ( \eta t )) \\
&\leq 2\tilde  C_{\kappa, \gamma, L, \alpha}\log(2/\delta)  \frac{\cB(\lambda) }{ \sqrt \lambda N^{\frac{\alpha}{2}} } 
\eta \left( 1 +  \bar \varphi(\eta T) \right)   \;.
\end{align*}
for some $\tilde  C_{\kappa, \gamma, L, \alpha} < \infty$ and where by Lemma \ref{lem:phi-bounds} for some $C_r < \infty$
\begin{equation}
\label{eq:phi}
 \bar \varphi(\eta T) :=  \frac{2}{T}\sum_{t= \floor{T/2} +1}^T\varphi ( \eta t ) \leq C_r \varphi(\eta T) \;. 
\end{equation} 
By \eqref{eq:A}, \eqref{eq:Bt-2}, \eqref{eq:xi2} and Lemma \ref{lem:phi-bounds} we get with $\lambda = (\eta T)^{-1}$ and $\eta < 1/\kappa^2$
\begin{align*}
 \mbe_{\hat D | D}[  \cI_1 ]  
&\leq \tilde C_{\kappa, \gamma, L, \alpha} \log(2/\delta)\frac{ \sqrt{\eta T} \cB(1/\eta T ) }{ N^{\frac{\alpha}{2}} } 
\frac{2\eta }{T} \sum_{t= \floor{T/2} +1}^T \sum_{s=0}^{t-2} \left( \frac{1}{e \eta (t-1-s)} + \lambda \right) (1+ \varphi(\eta s)) \\
&\leq 2\tilde C_{\kappa, \gamma, L, \alpha} \log(2/\delta)\frac{ \sqrt{\eta T} \cB(1/\eta T ) }{ N^{\frac{\alpha}{2}} }  
(4\log(T) + C'_r\log(T)\varphi(\eta T))  \;,
\end{align*}
for some $C'_r < \infty$.

Combining the bounds for $\cI_1$ and $\cI_2$ finally gives with $\eta < 1/\kappa^2$, $1 \leq \log(T)$ 
\begin{align*}
 \mbe_{\hat D | D}[ ||S_K ( \bar{ \hat f }_T - \bar f_T  ) ||_{L^2} ] &\leq \tilde  C \log(4/\delta) \frac{ \sqrt{\eta T}\cB(1/\eta T) }{ N^{\frac{\alpha}{2}} } 
\left( \log(T)  +  \log(T)\varphi(\eta T)   \right)  \;.
\end{align*}
for some $\tilde C < \infty$, depending on $\kappa, \gamma, L, \alpha, r$ and holding with probability at least $1-\delta$. 
\end{itemize}
\end{proof}


\subsection{Main Result Second Stage Tail-Averaged Gradient Descent}

We now derive the final error for the excess risk of the second-stage tail-ave GD estimator for tackling distribution regression. Recall that we have 
the decomposition
\begin{equation*}
\bar{ \hat f }_T - f_\rho =  
\underbrace{(\bar{ \hat f }_T - \bar{ f }_T )}_{GD \;\; Variance \;\; 2. \;\;stage} +  \underbrace{ (\bar{ f }_T - f_\rho )}_{Error \;\; GD \;\; first \;\; stage} \;. 
\end{equation*}
Our main results follows then immediately from Proposition \ref{prop:GD-usual} and Proposition \ref{prop:GD-variance}.

\vspace{0.5cm}


\begin{theorem}[Excess Risk Second-Stage tail-ave GD; Part I]
\label{theo:main-GD-partI}
Suppose Assumptions \ref{ass:bounded},  \ref{ass:n-big} are satisfied. Let additionally Assumptions \ref{ass:source} and \ref{ass:eigenvalue}  hold.   
Let $T \in \mbn$ and denote 
\[  B_T= \max\{2M, || S_K\bar u_T - f_\rho  ||_\infty\} \;. \]
Let further $\delta \in (0,1]$, $\lambda = (\eta T)^{-1}$, assume $0<r \leq 1$ and recall the definition of $ \cB(1/ \eta T )$ in \eqref{eq:B} 
and of $\varphi$ in \eqref{eq:small-phi}.  
With probability not less than $1-\delta$, 
the excess risk for the second-stage tail-averaged Gradient Descent satisfies:
\begin{enumerate}
\item {\bf If $1/2 \leq r \leq 1$: } 
\begin{align*}
\mbe_{\hat D | D}[ || S_K\bar{ \hat f }_T  - f_\rho ||_{L^2 }] &\leq C_1 \log(24/\delta )\frac{1}{\sqrt n}\;
   \left(  M \sqrt{ \cN(\lam)} +  \frac{||S_K \bar u_T - f_\rho||_{L^2}}{\sqrt \lam} + \frac{B_T}{\sqrt{n \lam }} \right) \\
   & +  C_2 \lam\; || (T_K + \lam)^{-1/2} \bar u_T ||_{\cH_K} +  C_3  ||S_K \bar u_T- f_\rho ||_{L^2} \\
   & + C_4 \log(8/\delta)
\log (T) \frac{\sqrt{\eta T}}{N^{\frac{\alpha}{2}}} \cB(1/ \eta T ) \;,
\end{align*}
for some constants $C_1>0$, $C_2>0$, $C_3>0$, $C_4>0$.
\item {\bf If $0<r \leq 1/2$: }
 \begin{align*}
\mbe_{\hat D | D}[ || S_K\bar{ \hat f }_T  - f_\rho ||_{L^2 }] &\leq C_1 \log(24/\delta )\frac{1}{\sqrt n}\;
   \left(  M \sqrt{ \cN(\lam)} +  \frac{||S_K \bar u_T - f_\rho||_{L^2}}{\sqrt \lam} + \frac{B_T}{\sqrt{n \lam }} \right) \\
   & +  C_2 \lam\; || (T_K + \lam)^{-1/2} \bar u_T ||_{\cH_K} +  C_3  ||S_K \bar u_T- f_\rho ||_{L^2} \\
& +   C_4 \log(8/\delta) \log(T) \frac{ \sqrt{\eta T}\cB(1/\eta T) }{ N^{\frac{\alpha}{2}} } 
\left( 1  +  \varphi(\eta T)   \right) \;,
\end{align*}
for some constants $C_1>0$, $C_2>0$, $C_3>0$, $C_4>0$.
\end{enumerate}
\end{theorem}

\vspace{0.5cm}

\begin{theorem}[Excess Risk Second-Stage tail-ave GD; Part II]
\label{theo:main-GD-partII}
Suppose Assumptions \ref{ass:bounded},  \ref{ass:n-big} are satisfied. Let additionally Assumptions \ref{ass:source} and \ref{ass:eigenvalue}  hold.   
Let $T \in \mbn$ and denote 
\[  B_T= \max\{2M, || S_K\bar u_T - f_\rho  ||_\infty\} \;. \]
Let further $\delta \in (0,1]$, $\lambda = (\eta T)^{-1}$, assume that $r \geq 1$ and recall the definition of $ \cB(1/ \eta T )$ in \eqref{eq:B}. 
Then with probability not less than $1-\delta$, 
the excess risk for the second-stage tail-averaged Gradient Descent satisfies 
\begin{align*}
\mbe_{\hat D | D}[ || S_K\bar{ \hat f }_T  - f_\rho ||_{L^2 }] &\leq C'_1\log(24/\delta )\frac{ 1}{\sqrt n}\;
   \left(   M  \sqrt{ \cN(\lam)} +  \frac{||S_K \bar u_T - f_\rho||_{L^2}}{\sqrt \lam} + \frac{B_T}{\sqrt{n \lam }} \right) \\
   & +  C'_2 \lam^{1/2}  \; || T_K^{-r } \bar u_T ||_{\cH_K}  \left(  \frac{\log(4/\delta)}{\sqrt n}   + \lam^{\zeta }\right) 
   +  C'_3 ||S_K \bar u_T - f_\rho ||_{L^2} \\
   & + C'_4 \log(8/\delta)
\log (T) \frac{\sqrt{\eta T}}{N^{\frac{\alpha}{2}}} \cB(1/ \eta T ) \;,
\end{align*}
for some constants $C'_1>0$, $C'_2>0$, $C'_3>0$, $C'_4>0$.
\end{theorem}

\vspace{0.5cm}

\begin{corollary}[Rate of Convergence Second-Stage Tail-Ave GD; mis-specified Case]
\label{cor:rate-dist-reg-miss}
Suppose all Assumptions of Proposition \ref{prop:GD-usual} are satisfied. Assume additionally that $r\leq 1/2$ and 
\[  n \geq  64 e \kappa^2 \log^2(24/\delta) (\eta t)\log((\eta T)^\nu) \;. \]
The excess risk for  second stage tail-averaged Gradient Descent satisfies in the mis-specified case with probability not less than $1-\delta$: 
\begin{enumerate}
\item If $2r+\nu > 1$, $\eta_n T_n = \left( \frac{R^2}{M^2}n\right)^{\frac{1}{2r + \nu}}$ and 
$N_n \geq \log^{2/\alpha}(n) \left(\frac{R^2 n}{M^2 }\right)^{\frac{2+\nu}{\alpha(2r+\nu)}}$, then  
\[  \mbe_{\hat D | D}[  || S_K\bar{ \hat f }_{T_n} - f_\rho ||_{L^2 } ] \leq C \; \log (24/\delta )  R\left(\frac{M^2}{R^2 n}\right)^{\frac{r}{2r + \nu}}  \;,\]
for some $C<\infty$, provided $n$ is sufficiently large. 
\item If $2r+\nu \leq 1$, $\eta_n T_n =  \frac{R^2n}{M^2 \log^K(n)}$ for some $K>1$ and 
$N_n \geq  \log^{2/\alpha}(n) \left(\frac{R^2 n}{M^2\log^K(n) }\right)^{\frac{2+\nu}{\alpha}}$, then  
\[  \mbe_{\hat D | D}[  || S_K\bar{ \hat f }_{T_n} - f_\rho ||_{L^2 } ] \leq C' \; \log (6/\delta ) R\left( \frac{M^2 \log^K(n)}{R^2n} \right)^r\;,\]
for some $C'<\infty$, provided $n$ is sufficiently large. 
\end{enumerate}
\end{corollary}

\begin{proof}[Proof of Corollary \ref{cor:rate-dist-reg-miss}]
Assume $0<r\leq 1/2$. 
\begin{enumerate}
\item Let  $2r+\nu > 1$ and $\eta_n T_n = \left( \frac{R^2}{M^2}n\right)^{\frac{1}{2r + \nu}}$. The first part of Corollary \ref{cor:rates-first-stage-GD} 
gives a rate for the first-stage GD of order 
\[  || S_K\bar f_T - f_\rho ||_{L^2 } \leq C \; \log (12/\delta ) R \left(\frac{M^2}{R^2n} \right)^{\frac{r}{2r + \nu}} \;, \]
provided $n$ is sufficiently large. It remains to bound the term  
\[  \log(T) \frac{ \sqrt{\eta T}\cB(1/\eta T) }{ N^{\frac{\alpha}{2}} } 
\left( 1  +  \varphi(\eta T)   \right)  \] 
from the second part of Theorem \ref{theo:main-GD-partI} for an appropriate choice on $N$. Note that in this case we have 
\[ \varphi(\eta T) = (\eta T)^{\frac{1}{2}\max\{\nu , 1-2r\}} =   (\eta T)^{\nu/2} \;. \]
Moreover, by the definition \eqref{eq:B}, the choice of $\eta_nT_n$ shows that for sufficiently large $n$ 
\begin{align*}
 \cB(1/\eta_n T_n) &\lesssim 1+ \frac{2\eta_n T_n}{n} + \sqrt{\frac{(\eta_n T_n)^{\nu+1}}{n}} +\frac{\eta_n T_n}{N^{\alpha/2}} \\
 &\lesssim 2+ \left( \frac{R^2}{M^2}n\right)^{\frac{1-2r}{2(2r+\nu)}} +  N^{-\alpha/2}\left( \frac{R^2}{M^2}n\right)^{\frac{1}{2r + \nu}} \;.
\end{align*}
Thus, letting $N_n\geq \left( \frac{R^2}{M^2}n\right)^{\frac{1+2r}{\alpha(2r + \nu)}}$ gives  for sufficiently large $n$ 
\begin{align*}
 \cB(1/\eta_n T_n) &\lesssim  2+ \left( \frac{R^2}{M^2}n\right)^{\frac{1-2r}{2(2r+\nu)}}\\
 &\lesssim \left( \frac{R^2}{M^2}n\right)^{\frac{1-2r}{2(2r+\nu)}}\;.
\end{align*}
Hence, give these choices, 
\begin{align*}
\log(T_n) \frac{ \sqrt{\eta_n T_n}\cB(1/\eta_n T_n) }{ N_n^{\frac{\alpha}{2}} } \left( 1  +  \varphi(\eta_n T_n)   \right) 
&\lesssim  \log(T_n)   N^{-\alpha/2}  \sqrt{\eta_n T_n}   \left( \frac{R^2}{M^2}n\right)^{\frac{1-2r}{2(2r+\nu)}}   (\eta_n T_n)^{\nu/2} \\
&\lesssim  N^{-\alpha/2}   \log(T_n)  \left( \frac{R^2}{M^2}n\right)^{\frac{2-2r+\nu}{2(2r+\nu)}} \;.
\end{align*}
Thus, if $N_n\geq \log^{2/\alpha}(n)\left( \frac{R^2}{M^2}n\right)^{\frac{2+\nu}{\alpha(2r + \nu)}}$ gives 
\[   N^{-\alpha/2}   \log(T_n)  \left( \frac{R^2}{M^2}n\right)^{\frac{2-2r+\nu}{2(2r+\nu)}} \lesssim   R \left(\frac{M^2}{R^2n} \right)^{\frac{r}{2r + \nu}} \;. \]
Hence, to obtain the given rate of convergence we need to choose 
\[ N_n\geq  \log^{2/\alpha}(n)  \left( \frac{R^2}{M^2}n\right)^{\beta} \;, \quad 
\beta = \max\left\{ \frac{1+2r}{\alpha(2r + \nu)} , \frac{2+\nu}{\alpha(2r + \nu)}\right\} = \frac{2+\nu}{\alpha(2r + \nu)} \;.  \]

\item Let $2r+\nu \leq 1$, $\eta_n T_n =  \frac{R^2n}{M^2 \log^K(n)}$ for some $K>1$. We again have to bound the expression 
\[  \log(T) \frac{ \sqrt{\eta T}\cB(1/\eta T) }{ N^{\frac{\alpha}{2}} } 
\left( 1  +  \varphi(\eta T)   \right) \]
for a suitable choice of $N$. Note that we have in this case 
\[ \varphi(\eta T) = (\eta T)^{\frac{1}{2}\max\{\nu , 1-2r\}} =   (\eta T)^{\frac{1}{2}-r} \;.  \] 
Moreover, for sufficiently large $n$ 
\begin{align*}
 \cB(1/\eta_n T_n) &\lesssim 1+ \frac{2\eta_n T_n}{n} + \sqrt{\frac{(\eta_n T_n)^{\nu+1}}{n}} +\frac{\eta_n T_n}{N^{\alpha/2}}  \\
 &\lesssim 2+ \left( \frac{R^2n}{M^2 \log^K(n)}\right)^{\nu/2}  + N^{-\alpha/2}  \frac{R^2n}{M^2 \log^K(n)}  \;.
\end{align*}
Thus, if $N_n \geq \left(\frac{R^2n}{M^2 \log^K(n)}\right)^{\frac{2-\nu}{\alpha}}$ we have 
\[ N_n^{-\alpha/2}  \frac{R^2n}{M^2 \log^K(n)} \lesssim  \left( \frac{R^2n}{M^2 \log^K(n)}\right)^{\nu/2} \]
and therefore 
\[  \cB(1/\eta_n T_n) \lesssim   \left( \frac{R^2n}{M^2 \log^K(n)}\right)^{\nu/2} \;. \]
We thus obtain for sufficiently large $n$ 
\begin{align*}
 \log(T_n) \frac{ \sqrt{\eta_n T_n}\cB(1/\eta_n T_n) }{ N_n^{\frac{\alpha}{2}} }\left( 1  +  \varphi(\eta_n T_n)   \right)  &\lesssim 
 \log(T_n) N_n^{-\alpha/2}  \left(\frac{R^2n}{M^2 \log^K(n)} \right)^{1+\nu/2-r} \;.
\end{align*}
Hence, with  $N_n \geq  \log^{2/\alpha}(n) \left(\frac{R^2n}{M^2 \log^K(n)}\right)^{\frac{2+\nu}{\alpha}}$ we find 
\[ \log(T_n) \frac{ \sqrt{\eta_n T_n}\cB(1/\eta_n T_n) }{ N_n^{\frac{\alpha}{2}} }\left( 1  +  \varphi(\eta_n T_n)   \right) 
\lesssim  \left( \frac{M^2 \log^K(n)}{R^2n} \right)^r \;. \]
\end{enumerate}
\end{proof}

\vspace{0.5cm}

\begin{corollary}[Rate of Convergence Second-Stage Tail-Ave GD; well-specified Case]
\label{cor:rate-dist-reg-well}
Suppose all Assumptions of Proposition \ref{prop:GD-usual} are satisfied. Assume additionally that $r\geq 1/2$ and 
\[  n \geq  64 e \kappa^2 \log^2(24/\delta) (\eta T)^{1+\nu} \;. \]
Let $\eta_n T_n =  \left( \frac{R^2}{M^2}n\right)^{\frac{1}{2r + \nu}}$ and 
$N_n \geq \log^{2/\alpha}(n)\left(\frac{R^2n}{M^2}\right)^{\frac{2r+1}{\alpha(2r+\nu)}}$. 
The excess risk for  second stage tail-averaged Gradient Descent satisfies in the well-specified case with probability not less than $1-\delta$ 
\[  \mbe_{\hat D | D}[  || S_K\bar{ \hat f }_{T_n} - f_\rho ||_{L^2 } ]\leq C \; \log (24/\delta ) R\left(\frac{M^2}{R^2 n}\right)^{\frac{r}{2r + \nu}} \;,\]
for some $C<\infty$, provided $n$ is sufficiently large. 
\end{corollary}

\begin{proof}[Proof of Corollary \ref{cor:rate-dist-reg-well}]
The proof follows the same lines as the proof of Corollary \ref{cor:rate-dist-reg-miss} and can be obtained from 
standard calculations.
\end{proof}


\subsection{Additional Material}

\begin{lemma}
\label{lem:phi-bounds}
Let $\varphi$ be defined by \eqref{eq:small-phi}.  
\begin{enumerate}
\item Let $ \bar \varphi(\eta T)$ be defined by \eqref{eq:phi}. For some $C_r \in \mbr_+$ we have the bound 
\[ \bar \varphi(\eta T) \leq C_r \varphi(\eta T) \;. \] 


\item  For some $C'_r \in \mbr_+$ we have the bound 
\[  \frac{2 }{T} \sum_{t= \floor{T/2} +1}^T \sum_{s=0}^{t-2} \frac{ \varphi ( \eta s ) }{t-1-s} \leq C'_r \log(T) \varphi(\eta T) \;. \]

\item With $\lambda = (\eta T )^{-1}$ we have 
\[ \frac{2\eta }{T} \sum_{t= \floor{T/2} +1}^T  \sum_{s=0}^{t-2}  \left( \frac{1}{e \eta (t-1-s)} + \lambda \right) \leq 4\log(T) \;.  \]

\item With $\lambda = (\eta T )^{-1}$ we have for some $C'_r < \infty$
\[ \frac{2\eta }{T} \sum_{t= \floor{T/2} +1}^T \sum_{s=0}^{t-2} \left( \frac{1}{e \eta (t-1-s)} + \lambda \right) (1+ \varphi(\eta s))
   \leq 4\log(T) + C'_r \log(T)\varphi(\eta T)  \;. \]
\end{enumerate}
\end{lemma}

\begin{proof}[Proof of Lemma \ref{lem:phi-bounds}]
\begin{enumerate}
\item Here we use the fact that for any $\alpha >0$, $1 \leq S \leq T$ 
\[ \sum_{t=S}^T t^\alpha \leq \int_{S}^{T+1} t^\alpha \; dt \leq \frac{2^{\alpha+1}}{\alpha+1}T^{\alpha +1} \;.\]
Hence,
\[ \frac{2}{T}  \sum_{t=\floor{T/2} + 1}^T t^\alpha \leq \frac{2^{\alpha+2}}{\alpha+1}T^{\alpha } \;.\]

\item Observe that for $\alpha \geq 0$
\[ \sum_{s=0}^{t-1}\frac{s^\alpha}{t-1-s} \leq 4 t^\alpha \log(t) \;. \]
Thus, by the first part of the Lemma we find 
\begin{align*}
 \frac{2 }{T} \sum_{t= \floor{T/2} +1}^T \sum_{s=0}^{t-2} \frac{ \varphi ( \eta s ) }{t-1-s}  &\leq \frac{8 C_r}{T} \sum_{t= \floor{T/2} +1}^T \log(t)\varphi ( \eta t ) \\
 &\leq 4 C_r \log(T) \bar \varphi ( \eta T ) \\
 &\leq C'_r \log(T) \varphi ( \eta T )  \;.
\end{align*} 

\item 
Note  that for any $t\geq 3$
\[   \sum_{s=0}^{t-2}   \frac{1}{ (t-1-s)} \leq 4 \; log(t) \]
\begin{align*}
\frac{2\eta }{T} \sum_{t= \floor{T/2} +1}^T  \sum_{s=0}^{t-2}  \left( \frac{1}{e \eta (t-1-s)} + \lambda \right) &= 
\frac{2\lambda \eta }{T} \sum_{t= \floor{T/2} +1}^T  \sum_{s=0}^{t-2} 1    + 
\frac{2}{eT} \sum_{t= \floor{T/2} +1}^T  \sum_{s=0}^{t-2} \frac{1}{t-1-s} \nonumber \\
&\leq \lambda \eta T + \frac{8}{eT} \sum_{t= \floor{T/2} +1}^T  log(t) \nonumber \\
&\leq \lambda \eta T + 2\log(T) \;. 
\end{align*}
The result follows by setting $\lambda = (\eta T )^{-1}$ and with $1\leq 2\log(T)$.
 
\item 
This follows immediately from the other parts of the Lemma. 
\end{enumerate}
\end{proof}



\section{Results for Tail-Averaged SGD}
\label{app:results-SGD-second}

This section is devoted to providing our final error bound for the second-stage SGD algorithm. Here, we write

\begin{equation}
\label{eq:error-final}
\mbe_{\hat D | D}[ || S_K\bar h_T - f_\rho ||_{L^2 }] \leq \underbrace{\mbe_{\hat D | D}[ || S_K\bar f_T - f_\rho ||_{L^2 }]}_{2. \;\; stage\;\; GD} + 
\underbrace{\mbe_{\hat D | D} [ S_K(   \bar{ \hat h }_T  -  \bar{ \hat f }_T ) ||_{L^2 } ] }_{2. \;\; stage\;\; SGD \;\; variance}  \;.
\end{equation}

\subsection{Bounding Second Stage SGD Variance}

A short calculation shows that the second stage SGD variance can be rewritten as 
\begin{equation*}
\hat h_{t+1} - \hat f_{t+1} = (Id - \eta \hat T_{t+1})(\hat h_{t} - \hat f_{t}) + \eta  \hat \zeta _{t+1} 
\end{equation*}
where we set $J_t:=\{ b(t-1)+1 , ..., bt\}$ and  define  
\[  \hat T_{t} := \frac{1}{b}\sum_{i \in J_t} K_{\mu_{\hat x_{j_i}}} \otimes  K_{\mu_{\hat x_{j_i}}} \;, 
\quad \hat g_t :=  \frac{1}{b}\sum_{i \in J_t}  y_{j_i}  K_{\mu_{\hat x_{j_i}}}\]
and 
\[  \hat \zeta _{t}  := (T_{\hat \bx} - \hat T_{t})\hat f_t  + ( \hat g_t - g_{\hat \bx}) \;. \]
This gives 
\[ \mbe_{J_t} [\hat \zeta _{t} | \hat D , D ] = 0 \]
and by Lemma 6 in \cite{mucke2019beating} we find 
\begin{equation}
\label{eq:noise:bound}
 \mbe_{J_t} [\hat \zeta _{t} \otimes \hat \zeta _{t}| \hat D , D ] \preceq 
\frac{1}{b} \left( \kappa^4 \sup_{ t }||\hat f_t||_{\cH_K}^2 + M^2 \right) T_{\hat \bx} \;.
\end{equation}

\vspace{0.3cm} 

As a preliminary step we need to bound the norm of the second stage GD updates.

\vspace{0.3cm}

\begin{proposition}
\label{lem:boundhatft}
Suppose Assumptions \ref{ass:bounded2},  \ref{ass:bounded}, \ref{ass:lipschitz}, \ref{ass:source} and \ref{ass:eigenvalue} are satisfied and let
$\eta < 1/\kappa^2$. 
\begin{enumerate}
\item {\bf If $f_\rho \in Ran(S_K)$: } Assume that 
\begin{equation}
\label{eq:nlarge2}
n \geq   64 e \kappa^2 \log^2(12/\delta) (\eta t)^{1+\nu}  \;.
\end{equation}
Then 
\[ \mbe_{\hat D | D}[ ||\hat f_{t+1}||_{\cH_K}^2  ]  \leq  C_{\alpha,\kappa, M,R} \left( \frac{\eta^2 (t+1)^2}{N^{\alpha}}  + 1 \right) \;, \]
with probability at least $1-\delta$ w.r.t. the data $D$, for some $ C_{\alpha,\kappa, M,R}< \infty$. 

\item {\bf If $f_\rho \not \in Ran(S_K)$: } 
Assume 
\[ n \geq   64 \kappa^2 \log(12/\delta) (\eta t) \log((\eta t)^\nu)  \;.  \]
With probability at least $1-\delta$ w.r.t. the data $D$ we have 
\[   \mbe_{\hat D | D}[ ||\hat f_{t+1}||_{\cH_K}^2  ]  \leq   C'_{\alpha,\kappa, M,R} \frac{\eta^2 (t+1)^2}{N^{\alpha}} \left(  1  +  \varphi^2(\eta t) \right)\;, \]
for some $ C'_{\alpha,\kappa, M,R}< \infty$  and where 
\begin{equation}
\label{eq:phi22}
\varphi (\eta t) = (\eta t)^{\frac{1}{2}\max\{ \nu , 1-2r \}}   \;. 
\end{equation}  
\end{enumerate}
\end{proposition}

\vspace{0.3cm}

\begin{proof}[Proof of Proposition \ref{lem:boundhatft}]
We split 
\begin{align}
\label{eq:split}
 ||\hat f_t||_{\cH_K}^2 &\leq 2|| \hat f_t - f_t ||_{\cH_K}^2 + 2||f_t  ||_{\cH_K}^2 \;.
\end{align}
According to \eqref{eq:diffx} we have 
\[ \hat f_{t+1} - f_{t+1} = \eta\sum_{s=0}^t (Id - \eta T_{\hat \bx} )^{t-s} \hat \xi_s  \;,  \]
where $ \xi_s $ is defined in \eqref{eq:noise}. We proceed by using convexity to obtain 
\begin{align*}
||\hat f_{t+1} - f_{t+1}||^2 &= \eta^2 (t+1)^2 \left| \left| \frac{1}{t+1} \sum_{s=0}^t (Id - \eta T_{\hat \bx} )^{t-s} \hat \xi_s    \right|\right|_{\cH_K}^2 \\
&\leq \eta^2 (t+1)  \sum_{s=0}^t \left| \left|   (Id - \eta T_{\hat \bx} )^{t-s} \hat \xi_s    \right|\right|_{\cH_K}^2 \\
&\leq \eta^2 (t+1)\sum_{s=0}^t || \hat \xi_s  ||_{\cH_K}^2 \;.  
\end{align*}
For bounding the noise variables we follow the proof of Proposition \ref{prop:GD-variance} and distinguish between the two cases: 
\begin{itemize}
\item {\bf $Pf_\rho \in Ran(S_K)$: } By Lemma \ref{lem:2}, Lemma \ref{lem:3} and Corollary \ref{prop:GD-bound2} we have 
\[   \mbe_{\hat D | D}[ ||\hat \xi_s ||^2_{\cH_K} ] \leq \frac{c_\alpha}{N^{\alpha}} \;, \]
for some $c_\alpha < \infty$ and holding with probability at least $1-\delta$ w.r.t. the data $D$, provided \eqref{eq:nlarge2} is satisfied. 
Thus, 
\begin{equation*}
\mbe_{\hat D | D}[ ||\hat f_{t+1} - f_{t+1}||_{\cH_K}^2  ] \leq c_\alpha \frac{\eta^2 (t+1)^2}{N^{\alpha}} 
\end{equation*}
in this case. Combining with \eqref{eq:split} and  Corollary \ref{prop:GD-bound2} once more  leads to 
\[ \mbe_{\hat D | D}[ ||\hat f_{t+1}||_{\cH_K}^2  ]  \leq  \tilde c_\alpha \left( \frac{\eta^2 (t+1)^2}{N^{\alpha}}  + 1 \right) \;, \]
with probability at least $1-\delta$ w.r.t. the data $D$, for some $\tilde c_\alpha < \infty$.

\item {\bf $Pf_\rho \not \in Ran(S_K)$: }  From Corollary \ref{prop:GD-bound2}, Lemma \ref{lem:1} and  Lemma \ref{lem:2} we obtain with 
probability at least $1-\delta$ with respect to the data $D$ 
\[ \mbe_{\hat D | D}[ ||\hat \xi_s ||^2_{\cH_K} ] \leq  \log^2(6/\delta)\frac{ \tilde c'_\alpha }{ N^{\alpha} } ( 1 + \varphi(\eta s)  )^2 \;, \] 
\end{itemize}
for some $c'_\alpha < \infty$. 
Thus, 
\begin{equation*}
\mbe_{\hat D | D}[ ||\hat f_{t+1} - f_{t+1}||_{\cH_K} ^2  ] \leq 2\tilde c'_\alpha  \frac{\eta^2 (t+1)}{N^{\alpha}} \sum_{s=0}^t (1 + \varphi^2(\eta s)) 
\end{equation*}
and by \eqref{eq:split}, since $\varphi$ is non-decreasing in $s$
\begin{align*}
 \mbe_{\hat D | D}[ ||\hat f_{t+1}||_{\cH_K}^2  ] &\leq 2\tilde c'_\alpha  \frac{\eta^2 (t+1)}{N^{\alpha}} \sum_{s=0}^t (1 + \varphi^2(\eta s))  
  +  c_{\kappa, M,R} \varphi^2(\eta t) \\
&\leq \tilde c''_{\alpha,\kappa, M,R}  \frac{\eta^2 (t+1)^2}{N^{\alpha}} \left(  1  +  \frac{1}{t+1} \sum_{s=0}^t \varphi^2(\eta s) \right)   \\
&\leq   \tilde c''_{\alpha,\kappa, M,R}  \frac{\eta^2 (t+1)^2}{N^{\alpha}} \left(  1  +  \varphi^2(\eta t) \right)   \;,  
\end{align*}  
for some $\tilde c''_{\alpha,\kappa, M,R}< \infty$ and with 
probability at least $1-\delta$ with respect to the data $D$. 
\end{proof}

\vspace{0.5cm}

\begin{proposition}[Second Stage SGD Variance]
\label{prop:SGD-var-second-stage}
Suppose Assumptions \ref{ass:bounded} and \ref{ass:lipschitz} are satisfied and let 
$\eta \kappa^2 < 1/4$, $\nu \in (0,1]$. Assume further that $Trace[T^\nu_{\hat \bx}] \leq C_\nu$ almost surely for some $C_\nu \in \mbr_+$. 
The second stage SGD variance satisfies with probability at least $1-\delta$ w.r.t. the data $D$
\begin{align*}
 \mbe_{\hat D | D} [ ||T_K^{1/2}(   \bar{ \hat h }_T  -  \bar{ \hat f }_T ) ||_{\cH_K} ]  &\leq  
\tilde C_{\nu, \kappa, M} 6\log(4/\delta)\;  \sqrt{ \frac{\eta}{b} (\eta T)^{\nu -1}}\;
  \left( 1 + \mbe_{\hat D | D}\left[   ||\hat f_T||_{\cH_K}^2  \right]^{1/2} \right)  \\
&    \left( \frac{\eta T}{\sqrt n} + 
 c_\alpha \gamma^{\alpha}LM \frac{\sqrt{\eta T}}{ N^{\frac{\alpha}{2}}}  + 1 \right)^{1/2}   \;,
\end{align*} 
for some $\tilde C_{\nu, \kappa, M} < \infty$.
\end{proposition}

\vspace{0.3cm}

\begin{proof}[Proof of Proposition \ref{prop:SGD-var-second-stage}]
H\"older's inequality allows us to write for any $\lambda >0$
\begin{align}
\label{eq:2}
  \mbe_{\hat D | D} [ ||T_K^{1/2}(   \bar{ \hat h }_T  -  \bar{ \hat f }_T ) ||_{\cH_K} ]  
&\leq   \left[ \mbe_{\hat D | D} [ ||T_K^{1/2} (T_{\hat \bx} + \lambda)^{-1/2}  ||^2 ] \right]^{\frac{1}{2}} 
   \left[ \mbe_{\hat D | D} [ ||(T_{\hat \bx} + \lambda)^{1/2}(   \bar{ \hat h }_T  -  \bar{ \hat f }_T ) ||_{\cH_K}^2 ] \right]^{\frac{1}{2}} \;.
\end{align}
For bounding the first term let us firstly observe that by Lemma \ref{lem:0} with probability at least $1-\delta$ 
\begin{align}
\label{eq:1}
\mbe_{\hat D | D} [ ||T_K^{1/2} (T_{\hat \bx} + \lambda)^{-1/2} ||^2 ]  &\leq ||T_K (T_{\hat \bx} + \lambda)^{-1} || \nonumber \\
&\leq   6\log(2/\delta)\frac{1}{\lambda \sqrt n} +  \frac{c_\alpha \gamma^{\alpha}LM}{\sqrt{\lam} N^{\frac{\alpha}{2}}}  + 1 \;. 
\end{align}
For bounding the second term we write 
\[ ||(T_{\hat \bx} + \lambda)^{1/2}(   \bar{ \hat h }_T  -  \bar{ \hat f }_T ) ||_{\cH_K}^2 =  ||T_{\hat \bx} ^{1/2}(   \bar{ \hat h }_T  -  \bar{ \hat f }_T ) ||_{\cH_K}^2 + 
 \lambda ||   \bar{ \hat h }_T  -  \bar{ \hat f }_T  ||_{\cH_K}^2 \;. \]
Applying Proposition 5 in \cite{mucke2019beating} with 
$\sigma^2 =  \frac{1}{b} \mbe_{\hat D | D} \left[ \kappa^4||\hat f_T||^2 + M^2 \right]$ then gives 
with $\lambda = (\eta T)^{-1}$ and for any $\nu \in (0,1]$
\begin{align*}
  \mbe_{\hat D | D} [ ||(T_{\hat \bx} + \lambda)^{1/2}(   \bar{ \hat h }_T  -  \bar{ \hat f }_T ) ||_{\cH_K}^2 ] &\leq   
C \frac{\eta}{b} (\eta T)^{\nu -1}  \mbe_{\hat D | D} \left[  \left( \kappa^4 ||\hat f_T||^2 + M^2 \right) Trace[T^\nu_{\hat \bx}] \right] \\
&\leq \tilde C_{\nu, \kappa, M} \frac{\eta}{b} (\eta T)^{\nu -1} \left( \mbe_{\hat D | D}[||\hat f_T||_{\cH_K}^2]  + 1\right) \;,
\end{align*}  
for some $\tilde C_{\nu, \kappa, M}<\infty$. Combining this with \eqref{eq:1} and \eqref{eq:2} finally leads to the result. 
\end{proof}

\vspace{0.3cm}

From  Proposition \ref{lem:boundhatft} and Proposition \ref{prop:SGD-var-second-stage} we immediately obtain: 

\vspace{0.3cm}

\begin{corollary}[Second Stage SGD Variance]
\label{cor:SGD-var-second-stage}
In addition to the Assumptions from Proposition \ref{prop:SGD-var-second-stage} suppose that Assumptions \ref{ass:source}, \ref{ass:eigenvalue}  are satisfied. 
\begin{enumerate}
\item {\bf If $Pf_\rho \in Ran(S_K)$: } 
Assume that 
\[ n \geq   64 e \kappa^2 \log^2(12/\delta) (\eta t)^{1+\nu}  \;. \]
Then 
\[ \mbe_{\hat D | D} [ ||T_K^{1/2}(   \bar{ \hat h }_T  -  \bar{ \hat f }_T ) ||_{\cH_K}  ]   \leq 
 C_{\nu, \kappa,\gamma,  M,\alpha, L} \; \left(  1 + \frac{\eta T}{\sqrt n} +  \frac{\sqrt{\eta T}}{ N^{\frac{\alpha}{2}}}   \right)^{1/2} 
\sqrt{ \frac{\eta}{b} (\eta T)^{\nu -1}}\;
  \left( 1 +   \frac{\eta T}{N^{\frac{\alpha}{2}}}  \right)  \;, \]
with probability at least $1-\delta$ w.r.t. the data $D$, for some $C_{\nu, \kappa,\gamma,  M,\alpha, L} < \infty$. 

\item {\bf If $Pf_\rho \not \in Ran(S_K)$: } 
Assume that
\[  n \geq  64 e \kappa^2 \log^2(24/\delta) (\eta t)\log((\eta T)^\nu) \;. \]
Then, with probability at least $1-\delta$ w.r.t. the data $D$ we have 
\[   \mbe_{\hat D | D} [ ||T_K^{1/2}(   \bar{ \hat h }_T  -  \bar{ \hat f }_T ) ||_{\cH_K}  ]  \leq   
\tilde  C_{\nu, \kappa,\gamma,  M,\alpha, L}\log(6/\delta) \; \left( 1 + \frac{\eta T}{\sqrt n} +  \frac{\sqrt{\eta T}}{ N^{\frac{\alpha}{2}}}   \right)^{1/2} 
 \sqrt{ \frac{\eta}{b} (\eta T)^{\nu -1}}\;
  \left( 1 +   \varphi(\eta T) \frac{\eta T}{N^{\frac{\alpha}{2}}}   \right)
\;, \]
 for some $\tilde C_{\nu, \kappa,\gamma,  M,\alpha, L} < \infty$ and where $\varphi$ is defined in \eqref{eq:phi22}. 
\end{enumerate}
\end{corollary}


\subsection{Main Result Second Stage Tail-Averaged SGD }

Combining now \eqref{eq:error-final} with  Proposition \ref{prop:SGD-var-second-stage},  Theorem \ref{theo:SGD-PartI} and Theorem \ref{theo:main-GD-partII} 
finally leads to our main results.

\vspace{0.3cm}

\begin{theorem}[Excess Risk Second-Stage tail-ave GD; Part I]
\label{theo:SGD-PartI}
Suppose Assumptions \ref{ass:bounded},  \ref{ass:n-big} are satisfied. Let additionally Assumptions \ref{ass:source} and \ref{ass:eigenvalue}  hold.   
Let $T \in \mbn$ and denote 
\[  B_T= \max\{2M, || S_K\bar u_T - f_\rho  ||_\infty\} \;. \]
Let further $\delta \in (0,1]$, $\lambda = (\eta T)^{-1}$, $\eta \kappa^2 < 1/4$, assume $0<r \leq 1$ and recall the definition of $ \cB(1/ \eta T )$ in \eqref{eq:B} 
and of $\varphi$ in \eqref{eq:small-phi}.  
With probability not less than $1-\delta$, 
the excess risk for the second-stage tail-averaged SGD satisfies:
\begin{enumerate}
\item {\bf If $1/2 \leq r \leq 1$: } 
\begin{align*}
\mbe_{\hat D | D}[ || S_K\bar{ \hat h }_T  - f_\rho ||_{L^2 }] &\leq C_1 \log(24/\delta )\frac{1}{\sqrt n}\;
   \left(  M \sqrt{ \cN(\lam)} +  \frac{||S_K \bar u_T - f_\rho||_{L^2}}{\sqrt \lam} + \frac{B_T}{\sqrt{n \lam }} \right) \\
   & +  C_2 \lam\; || (T_K + \lam)^{-1/2} \bar u_T ||_{\cH_K} +  C_3  ||S_K \bar u_T- f_\rho ||_{L^2} \\
   & + C_4 \log(8/\delta) \log (T) \frac{\sqrt{\eta T}}{N^{\frac{\alpha}{2}}} \cB(1/ \eta T ) \\
& + C_5\log(8/\delta) 
 \sqrt{ \frac{\eta}{b} (\eta T)^{\nu -1}}\;
  \left( 1 + \mbe_{\hat D | D}\left[   ||\hat f_T||_{\cH_K}^2  \right]^{1/2} \right) \;
 \left( \frac{\eta T}{\sqrt n} +  \frac{\sqrt{\eta T}}{ N^{\frac{\alpha}{2}}}  + 1 \right)^{1/2}       \;,
\end{align*}
for some constants $C_1>0$, $C_2>0$, $C_3>0$, $C_4>0$, $C_5>0$.
\item {\bf If $0<r \leq 1/2$: }
 \begin{align*}
\mbe_{\hat D | D}[ || S_K\bar{ \hat h }_T  - f_\rho ||_{L^2 }] &\leq C_1 \log(24/\delta )\frac{1}{\sqrt n}\;
   \left(  M \sqrt{ \cN(\lam)} +  \frac{||S_K \bar u_T - f_\rho||_{L^2}}{\sqrt \lam} + \frac{B_T}{\sqrt{n \lam }} \right) \\
   & +  C_2 \lam\; || (T_K + \lam)^{-1/2} \bar u_T ||_{\cH_K} +  C_3  ||S_K \bar u_T- f_\rho ||_{L^2} \\
& +   C_4 \log(8/\delta) \log(T) \frac{ \sqrt{\eta T}\cB(1/\eta T) }{ N^{\frac{\alpha}{2}} } 
\left( 1  +  \varphi(\eta T)   \right) \\
& + C_5\log(8/\delta) 
 \sqrt{ \frac{\eta}{b} (\eta T)^{\nu -1}}\;
  \left( 1 + \mbe_{\hat D | D}\left[   ||\hat f_T||_{\cH_K}^2  \right]^{1/2} \right) \;
 \left( \frac{\eta T}{\sqrt n} +  \frac{\sqrt{\eta T}}{ N^{\frac{\alpha}{2}}}  + 1 \right)^{1/2}     \;,
\end{align*}
for some constants $C_1>0$, $C_2>0$, $C_3>0$, $C_4>0$, $C_5>0$.
\end{enumerate}
\end{theorem}

\[\]

\begin{theorem}[Excess Risk Second-Stage tail-ave GD; Part II]
\label{theo:SGD-PartII}
Suppose Assumptions \ref{ass:bounded},  \ref{ass:n-big} are satisfied. Let additionally Assumptions \ref{ass:source} and \ref{ass:eigenvalue}  hold.   
Let $T \in \mbn$ and denote 
\[  B_T= \max\{2M, || S_K\bar u_T - f_\rho  ||_\infty\} \;. \]
Let further $\delta \in (0,1]$, $\lambda = (\eta T)^{-1}$, $\eta \kappa^2 < 1/4$, assume that $r \geq 1$ and recall the definition of $ \cB(1/ \eta T )$ in \eqref{eq:B}. 
Then with probability not less than $1-\delta$, 
the excess risk for the second-stage tail-averaged SGD satisfies 
\begin{align*}
\mbe_{\hat D | D}[ || S_K\bar{ \hat h }_T  - f_\rho ||_{L^2 }] &\leq C'_1\log(24/\delta )\frac{ 1}{\sqrt n}\;
   \left(   M  \sqrt{ \cN(\lam)} +  \frac{||S_K \bar u_T - f_\rho||_{L^2}}{\sqrt \lam} + \frac{B_T}{\sqrt{n \lam }} \right) \\
   & +  C'_2 \lam^{1/2}  \; || T_K^{-r } \bar u_T ||_{\cH_K}  \left(  \frac{\log(4/\delta)}{\sqrt n}   + \lam^{\zeta }\right) 
   +  C'_3 ||S_K \bar u_T - f_\rho ||_{L^2} \\
   & + C'_4 \log(8/\delta)
\log (T) \frac{\sqrt{\eta T}}{N^{\frac{\alpha}{2}}} \cB(1/ \eta T ) \\
& + C'_5\log(8/\delta) 
 \sqrt{ \frac{\eta}{b} (\eta T)^{\nu -1}}\;
  \left( 1 + \mbe_{\hat D | D}\left[   ||\hat f_T||_{\cH_K}^2  \right]^{1/2} \right) \;
 \left( \frac{\eta T}{\sqrt n} +  \frac{\sqrt{\eta T}}{ N^{\frac{\alpha}{2}}}  + 1 \right)^{1/2}   \;,
\end{align*}
for some constants $C'_1>0$, $C'_2>0$, $C'_3>0$, $C'_4>0$, $C'_5>0$.
\end{theorem}

\[\]

\begin{corollary}[Learning Rates Second Stage Ave-SGD Mis-Specified Model]
\label{cor:SGD-mis}
Suppose all assumptions of Theorem \ref{theo:SGD-PartI} and Theorem \ref{theo:SGD-PartII} are satisfied. 
Assume additionally that $r \leq 1/2$, $K>1$, $\eta_0 < \frac{1}{4\kappa^2}$ and 
\[  n \geq  64 e \kappa^2 \log(4/\delta) (\eta T) \log((\eta T)^\nu) \;. \]
\begin{enumerate}
\item Let  $2r+\nu >1$. 
Then, for any $n$ 
sufficiently large, the excess risk satisfies with probability at least $1-\delta$ w.r.t. the data $D$ 
\[  \mbe_{\hat D | D}[ || S_K\bar{ \hat h }_{T_n} - f_\rho ||_{L^2 }]^2  \leq C \log(24/\delta) R \left(\frac{M^2}{R^2 n}\right)^{\frac{r}{2r + \nu}} \;, \]
provided $N_n \geq \log^{2/\alpha}(n)\left(\frac{R^2}{M^2}n \right)^{\frac{2+\nu}{\alpha(2r + \nu)}}$ and 
\begin{itemize}
\item  Multi-pass SGD: $b_n=\sqrt{n}$, $\eta_n= \eta_0$ and $T_n=\left(\frac{R^2 n}{\sigma^2}\right)^{\frac{1}{2r+\nu}}$, 
\item Batch GD: $b_n=n$, $\eta_n= \eta_0$ and $T_n=\left(\frac{R^2 n}{\sigma^2}\right)^{\frac{1}{2r+\nu}}$.  
\end{itemize}

\item Let  $2r+\nu \leq 1$. 
Then, for any $n$ sufficiently large, the excess risk satisfies with probability at least $1-\delta$ w.r.t. the data $D$ 
\[  \mbe_{\hat D | D}[ || S_K\bar{ \hat h }_{T_n} - f_\rho ||_{L^2 }]^2  \leq C \log(24/\delta)R\left( \frac{M^2 \log^K(n)}{R^2n} \right)^r \;, \]
provided $N_n\geq \log^{2/\alpha}(n) \left(\frac{R^2n}{M^2\log^K(n)}\right)^{\frac{3-2r}{\alpha}}$ and 
\begin{itemize}
\item $b_n=1$, $\eta_n=\left( \frac{M^2\log^K(n)}{R^2n} \right)^{2r+\nu}$ and $T_n=\left( \frac{R^2n}{M^2\log^K(n)} \right)^{2r+\nu+1}$,  
\item $b_n=\left( \frac{R^2n}{M^2\log^K(n)} \right)^{2r+\nu}$, $\eta_n=\eta_0$ and $T_n=\frac{R^2n}{M^2\log^K(n)} $, 
\item $b_n= \frac{R^2n}{M^2\log^K(n)}$, $\eta_n=\eta_0$ and $T_n=\frac{R^2n}{M^2\log^K(n)} $.
\end{itemize}
\end{enumerate}
\end{corollary}

\begin{proof}[Proof of Corollary \ref{cor:SGD-mis}]
Here, we combine the results from Corollary \ref{cor:rate-dist-reg-miss} and Corollary \ref{cor:SGD-var-second-stage}. We have to show that 
\[  \underbrace{\left( 1 + \frac{\eta T}{\sqrt n} +  \frac{\sqrt{\eta T}}{ N^{\frac{\alpha}{2}}}  \right)^{1/2}}_{ \cT_1} \;
    \underbrace{\left( 1 +   \varphi(\eta T) \frac{\eta T}{N^{\frac{\alpha}{2}}}   \right) }_{\cT_2 }\;
    \underbrace{\sqrt{ \frac{\eta}{b} (\eta T)^{\nu -1}}}_{\cT_3} \]
is of optimal order under appropriate choices of all parameters. 
\begin{enumerate}
\item Let  $2r+\nu >1$ and $\eta_n T_n = \left( \frac{R^2}{M^2}n\right)^{\frac{1}{2r + \nu}}$. Given this choices, one easily verifies  that the leading order 
term in $\cT_1$ is given by $\left(\frac{\eta_n T_n}{\sqrt n}\right)^{1/2}$, provided 
\[ N_n \geq \left(\frac{n}{\eta_n T_n} \right)^{1/\alpha} \sim \left( \frac{R^2}{M^2}n\right)^{\frac{2r+\nu -1}{\alpha(2r + \nu)}} \;. \]
Moreover, we have $\varphi(\eta T) = (\eta T)^{\nu/2}$ and the second term $\cT_2$ is of order 1 if 
\[ (\eta_n T_n)^{1+\nu/2} N^{-\frac{\alpha}{2}} \lesssim 1\; \] 
hence if 
\[ N_n \geq  \left( \frac{R^2}{M^2}n\right)^{\frac{2+\nu}{\alpha(2r + \nu)}} \;, \]
for $n$ sufficiently large. Note that we have 
\[ \max\left\{  \frac{2r+\nu -1}{\alpha(2r + \nu)} ,  \frac{2+\nu}{\alpha(2r + \nu)}  \right\} =  \frac{2+\nu}{\alpha(2r + \nu)}  \;.\]
Finally, we have to determine now appropriate values of $\eta_n, T_n, b_n$ such that 
\[  \left(\frac{M^2}{R^2 n} \right)^{\frac{2r+\nu-2}{2(2r+\nu)}} \; \;\frac{\eta_n}{b_n} \left( \frac{R^2}{M^2}n\right)^{\frac{\nu-1}{2r + \nu}} 
\lesssim  R \left(\frac{M^2}{R^2 n}\right)^{\frac{2r}{2r + \nu}}\;,  \]
that is, if 
\[ \frac{\eta_n}{b_n}  \lesssim R \left(\frac{M^2}{R^2 n}\right)^{\frac{1}{2}} \;.  \] 
This is surely satisfied by all the given choices.

\item Let  $2r+\nu \leq 1$ and $\eta_n T_n =  \frac{R^2n}{M^2 \log^K(n)}$ for some $K>1$. Again, the leading order 
term in $\cT_1$ is given by $\left(\frac{\eta_n T_n}{\sqrt n}\right)^{1/2}$, provided 
\[ N^{\alpha/2}_n \geq \sqrt{\frac{n}{\eta_n T_n}} \sim \log^{K/2}(n)\;,  \]
or equivalently, 
\[ N_n \geq   \log^{K/\alpha}(n)  \;.\]
For bounding $\cT_2$ note that $\varphi(\eta T) = (\eta T)^{\frac{1}{2}-r}$. Then, $\cT_2$ is of order 1 if 
\[ N_n\geq  \left(\frac{R^2n}{M^2\log^K(n)}\right)^{\frac{3-2r}{\alpha}}\;.  \]
Finally, we have to determine now appropriate values of $\eta_n, T_n, b_n$ such that 
\[   \frac{R^2 n}{M^2\log^{K}(n)} \; \;\frac{\eta_n}{b_n} \left( \frac{R^2n}{M^2 \log^K(n)} \right)^{\nu-1} 
\lesssim  R\left( \frac{M^2 \log^K(n)}{R^2n} \right)^{2r} \;,  \]
that is, if 
\[ \frac{\eta_n}{b_n}  \lesssim R\left( \frac{M^2 \log^K(n)}{R^2n} \right)^{2r+\nu}   \;.  \] 
This is surely satisfied by all the given choices. 
\end{enumerate}
\end{proof}

\[\]

\begin{corollary}[Learning Rates Second Stage Ave-SGD Well-Specified Model]
\label{cor:SGD-well}
Suppose all assumptions of Theorem \ref{theo:SGD-PartI} and Theorem \ref{theo:SGD-PartII} are satisfied. Assume additionally that $r \geq \frac{1}{2}$ and 
\[  n \geq  64 e \kappa^2 \log(4/\delta) (\eta T)^{1+\nu} \;. \]
Let $\eta_0 < \frac{1}{4\kappa^2}$ and choose $N_n\geq  \log^{2/\alpha}(n)\left(\frac{R^2n}{\sigma^2}\right)^{\frac{2r+1}{\alpha(2r+\nu)}}$. 
Then, for any $n$ 
sufficiently large, the excess risk satisfies with probability at least $1-\delta$ w.r.t. the data $D$ 
\[  \mbe_{\hat D | D}[ || S_K\bar{ \hat h }_{T_n} - Pf_\rho ||_{L^2 }]^2  \leq C \log(24/\delta) R \left(\frac{\sigma^2}{R^2 n}\right)^{\frac{r}{2r + \nu}} \;, \]
for each of the following choices: 
\begin{enumerate}
\item One-pass SGD: $b=1$, $\eta_n= \eta_0 \frac{R^2}{\sigma^2}\left(\frac{\sigma^2}{R^2 n}\right)^{\frac{2r+\nu-1}{2r+\nu}}$ and $T_n = \frac{R^2}{\sigma^2}n$, 
\item Early stopping and one-pass SGD:  $b=n^{\frac{2r+\nu-1}{2r+\nu}}$, $\eta_n=\eta_0$ and $T_n=\left(\frac{R^2 n}{\sigma^2}\right)^{\frac{1}{2r+\nu}}$, 
\item Batch-GD: $b=n$, $\eta_n=\eta_0$ and $T_n=\left(\frac{R^2 n}{\sigma^2}\right)^{\frac{1}{2r+\nu}}$. 
\end{enumerate}
\end{corollary}

\begin{proof}[Proof of Corollary \ref{cor:SGD-well}]
The proof follows the same lines as the proof of  Corollary \ref{cor:SGD-mis} by standard calculations. 
\end{proof}



\end{document}